\documentclass{bmvc2k}
\usepackage{amssymb}
\usepackage{tikz}
\usepackage{pgfplots}
\usepackage{amsthm}
\usepackage{booktabs}
\usepackage{pgfplots}
\usepackage{array}
\usepackage{booktabs}
\usepackage{multirow}
\usepackage{pifont}
\usepackage{makecell}
\newcommand{\cmark}{\ding{51}}
\newcommand{\xmark}{\ding{55}}

\usetikzlibrary{calc,arrows.meta,decorations.text}
\usetikzlibrary{positioning, calc, arrows.meta} 
\definecolor{errorA}{RGB}{210, 140, 0}    
\definecolor{errorB}{RGB}{190, 50, 100}
\usepackage{xcolor}
\usepackage[normalem]{ulem}

\newcommand{\categoryerror}[1]{{\color{errorA}\uline{#1}}}
\newcommand{\temporalerror}[1]{{\color{errorB}\uline{#1}}}
\newcommand{\greenbordered}[1]{%
  \tikz\node[inner sep=1.5pt, draw=errorA, line width=1pt]{#1};%
}
\newcommand{\redbordered}[1]{%
  \tikz\node[inner sep=1.5pt, draw=errorB, line width=1pt]{#1};%
}

\usepackage{algorithm}
\usepackage{tcolorbox}  
\usepackage{cancel}
\usepackage{algpseudocode}
\newtheorem{definition}{Definition}
\newtheorem{theorem}{Theorem}
\newtheorem{lemma}[theorem]{Lemma}
\newtheorem{proposition}{Proposition}
\newtheorem{corollary}[theorem]{Corollary}

\title{Composed Historical Image Retrieval by Modeling Temporal Representations}

\addauthor{Adrià Molina}{amolina@cvc.uab.cat}{1,2}
\addauthor{Oriol Ramos Terrades}{oriolrt@cvc.uab.cat}{1,2}
\addauthor{Josep Lladós}{josep@cvc.uab.cat}{1,2}

\addinstitution{
 Centre de Visió per Computador\\
 Universitat Autònoma de Barcelona\\
 Bellaterra, Catalonia
}
\addinstitution{
 Computer Science Department\\
 Universitat Autònoma de Barcelona,\\
 Bellaterra, Catalonia
}

\runninghead{A. Molina \etal}{Composed Historical Image Retrieval}

\def\eg{\emph{e.g}\bmvaOneDot}

\def\etal{\emph{et al}\bmvaOneDot}
\definecolor{myTeal}  {HTML}{008080}
\definecolor{myAmber} {HTML}{C8872A}
\definecolor{myViolet}{HTML}{7B3F9E}
\begin{document}

\maketitle

\begin{abstract}
While time evolves linearly, the geometry of neural embedding spaces is inherently multi-dimensional, often chaotic, and difficult to interpret. In principle, one could constrain an embedding space to a single temporal dimension; however, such a reduction would sacrifice performance on downstream tasks, as one-dimensional embeddings cannot retain sufficient expressive capacity. This paper asks whether it is possible to learn representations that preserve temporal structure while remaining effective for image and object retrieval, and answers this question by building the mathematical foundations of such a system. We propose Temporally Decomposable Image Representations (TDIR), a representation learning algorithm that decomposes historical photographs into separate date and content components through orthogonal subspaces. We define and prove the conditions under which such a decomposition is achievable, characterize the error incurred when those conditions are only partially met, and show that orthogonality between temporal and categorical subspaces emerges naturally from the joint optimization, without requiring it to be imposed explicitly. Beyond its geometric properties, TDIR enables a class of transitive operations on embedding spaces: the temporal information of one image can be extracted and injected into the representation of another, with no label supervision required. All theoretical properties are grounded and validated in the real-world problem of Composed Image Retrieval on historical photographs, where a query simultaneously specifies object content and a target time period, either through labels or through example images. This in-the-wild setting serves as a concrete backing for the propositions we derive, offering an intuitive and interpretable way to navigate photographic archives while maintaining competitive performance in both date estimation and object retrieval.

\end{abstract}

\section{Introduction}
\label{sec:intro}
Despite photography emerging late during the 19th century, it still comprises an important portion of archival data. More precisely, around 5-10\% of archival material is estimated to be non-textual (see Appendix \ref{sec:how_many_photos}), despite covering a significantly shorter historical span than printed or handwritten documents. Unlike traditional historical records, however, photographs encode most of their information visually, making their description through discrete metadata inherently restrictive. In this context, many Digital Humanities projects aim to improve access to non-textual collections through semantic search \cite{yuan2025knowledge} or automatic metadata completion \cite{garcia2020contextnet}. Among such metadata descriptors, the date of a photograph is one of the most prominent task in historical photography analysis \cite{palermo2012dating,muller2017picture,stacchio2020imago,paplham2026photo,barancova2023blind}, as it enables historians and social scientists to contextualise much of the information contained in the image.

Temporal information alone, however, is insufficient for a complete archival description. Archivists must also contextualise the depicted objects within their historical period. In archival science terms \cite{peterson2014probative}, the preservation interest of a document is partially determined by its \textit{probative value}: the capacity of the image to uniquely illustrate an event, person, technology, or phenomenon. This requires understanding not only \textit{when} a photograph was taken, but also \textit{how} the objects it contains relate to their historical moment. This archival reasoning has direct implications for image analysis. Methods for automatic date estimation have evolved from low-level visual predictors \cite{palermo2012dating} toward semantic and object-centric approaches \cite{net2024transformer}. We argue that the latter is closer to the reasoning process of archivists, who infer dates through the historical signatures of the objects appearing in the scene. Objects such as cars, clothes, or posters evolve at different temporal rhythms, allowing trained observers to identify inconsistencies or confirming regularities at a glance.

This object-sensitive reasoning, however, is difficult to replicate in standard neural representations, where colour, texture, objects, and date tend to be entangled in ways that obscure the individual temporal evolution of each category. More critically for archival practice, current retrieval systems do not allow a user to query an archive by \textit{composing} an object of interest with a target time period, as an archivist naturally would. As illustrated in Figure~\ref{fig:visabs}, the ideal system would support queries such as: \textit{``find images of jackets from the 1940s''}, or more powerfully, \textit{``find images of cars that look contemporaneous to this photograph of a politician''} --- without requiring any date label for the reference image.

\begin{figure}[t]
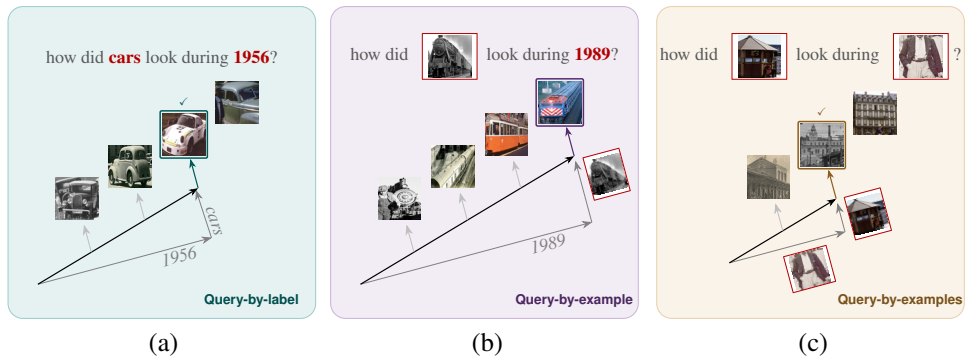

    \centering
    \setlength{\tabcolsep}{2pt} 

    \begin{tabular}{ccc}
        \resizebox{!}{0.32\textwidth}{\input{images/visual_abstract/visual_a}} &
        \resizebox{!}{0.32\textwidth}{\input{images/visual_abstract/visual_b}} &
        \resizebox{!}{0.32\textwidth}{\input{images/visual_abstract/visual_c}} \\
        (a) & (b) & (c)
    \end{tabular}

    \caption{Visual abstract of the three proposed representation objectives, showing composed image retrieval for query-by-label (a), query-by-example (b) and query-by-examples (c), in the latter, the idea is to inject the temporal information of an image into the object representation of the other query.}
    \label{fig:visabs}
\end{figure}

To enable this, we propose \textbf{Temporally Decomposable Image Representations (TDIR)},
a formal framework for representation learning in which the temporal and categorical 
components of an image embedding are separated into orthogonal subspaces. 
We formalise the conditions under which such a decomposition holds, prove that the 
required orthogonality emerges naturally from joint optimisation, and characterise 
the error incurred when those conditions are only partially met. 
These theoretical contributions are, to our knowledge, the first formal treatment of 
temporal decomposability in visual embeddings and stand independently of any specific 
application. Concretely, we establish two formal properties:

\begin{itemize}
    \item \textbf{Temporally Decomposable Image Representation:} an embedding space is 
    temporally decomposable when any representation can be expressed as a linear 
    combination of independent category and year vectors, such that each 
    component exclusively encodes its respective factor of variation.
    \item \textbf{Joint Proxy Optimization:} We propose a joint optimization which promotes TDIR in real case scenarios. Thus, the year centroid, interpreted as a displacement 
    vector, can transport any representation to a  different date without 
    altering its categorical content.
\end{itemize}

Beyond their theoretical interest, both properties admit a practically 
relevant instantiation: Composed Image Retrieval on in-the-wild historical photographs, where a 
query simultaneously specifies object content and a target time period. This provides empirical evidence that the derived properties yield interpretable behavior in applied settings beyond the theoretical formulation of the framework. 

\section{Related Work}
Early approaches to photographic date estimation relied on low-level visual cues such as colour histograms and film grain \cite{palermo2012dating}, while more recent methods leverage semantic features. Crucially, however, none of these works formalise the desired properties of an ideal temporal representation. Müller et al. introduce the DEW benchmark and treat date estimation as a regression problem \cite{muller2017picture}, without characterising the geometric structure that makes a representation well-suited for this task. Post-hoc analyses have revealed that temporal structure does emerge in general-purpose embeddings: In \cite{molina2022generic}, the authors propose a loss function that can re-organise vision embeddings according to a fixed temporal criteria, this temporal geometry is observed to naturally emerge through CLIP pretraining \cite{sonthalia2026rankability}, where embeddings carry rankable temporal information. Neither work describes how to induce such structure by design, nor what formal properties an ideally disentangled temporal representation should satisfy. TDIR addresses precisely this gap. A complementary line of work argues that date estimation should be grounded in the objects present in a scene, mirroring the reasoning of trained archivists. Ashida et al. \cite{ashida2021determining} propose an object-centred ensemble restricted to human subjects, and Net et al. \cite{net2024transformer} extend this idea to a broader set of object categories via a transformer-based architecture trained on DEW, introducing the DEW-B benchmark. However, both works treat object-centricity as a means to improve date prediction, not as a representational goal in its own right: object identity is discarded once the date is estimated. Furthermore, since \cite{net2024transformer} requires training a dedicated image encoder per object category, the method is evaluated on only six categories, which raises scalability concerns to larger object sets. By contrast, TDIR jointly preserves both categorical and temporal factors within a single unified backbone, enabling retrieval along either dimension independently and scaling naturally to a large number of object categories. Composed Image Retrieval (CIR), which involves querying by combining a reference image with a modification attribute, has been studied in general-domain settings \cite{wu2021fashion,ray2023cola}, but existing historical benchmarks do not support this paradigm. DEW provides date annotations without object labels \cite{muller2017picture}; IMAGO and yearbook-based datasets \cite{salem2016analyzing,stacchio2020imago,paplham2026photo} are restricted to faces and thus offer only a single semantic category. EUFCC-CIR proposed a composed retrieval dataset for cultural heritage collections \cite{net2024eufcc}, but it contains no temporal metadata, making category-date composition impossible. No existing benchmark simultaneously provides object-category and year-level annotations over a diverse photographic archive.
\section{Methodology}
\subsection{Decomposable and Transitive Temporal Representations}

Training operates on tuples $T = (x_A, x_B)$ where each image $x$ carries two 
attributes: a category $c \in \{1,\dots,C\}$ and a year $y \in \{1,\dots,Y\}$. In the tuples, $c_A = c_B = c$ and $y_A \neq y_B$. The geometric construction of our method rests on a key structural requirement for the year subspace: that temporal displacements between dates behave consistently across all categories.

\begin{definition}[Temporal Independence of Classes]
\label{def:definition_classes_assumption}
Given class and date centroids ($\mu^c$, $\mu^y$) and their corresponding labels ($c$ and $y$), a given representation is said to be temporally independent from the object classes when
\begin{equation}
p(\mu^c, c, \mu^y, y) = p(\mu^c, c) \cdot p(\mu^y, y)
\end{equation}
\end{definition}

From this definition, the pair $(\mu^c, c)$ is independent of $(\mu^y, y)$: category centroids carry no information about dates, and date centroids carry no information about categories.

\begin{definition}[Temporal Transitivity]
\label{def:temporal_trans}
For any ordered pair of dates $(A, B)$, define the \textbf{displacement vector}
$\Delta^y_{A \to B} = \mu^{y_B} - \mu^{y_A} \in \mathbb{R}^d$.
The year subspace is \textbf{temporally transitive} if, for any triple $(A, B, C)$,
\begin{equation}
\Delta^y_{A \to B} + \Delta^y_{B \to C} = \Delta^y_{A \to C}.
\end{equation}
\end{definition}

This property is necessary for zero-label inference: it guarantees that a year residual extracted from \textit{any} image acts as a consistent temporal operator when transplanted onto another image. We now introduce the proxy and residual structures that make this property learnable.

\subsection{Proxy Sets}
\label{sec:proxies}

Let $f : \mathcal{X} \to \mathbb{R}^d$ be a CNN encoder. Each image $x$ carries two attributes: a category $c \in \{1,\dots,C\}$ and a year $y \in \{1,\dots,Y\}$. We seek a representation where these attributes occupy \textit{orthogonal subspaces}, admitting the decomposition:
\begin{equation}
v = f(x) = \mu^c + \mu^y, 
\end{equation}
where $\mu^c, \mu^y \in \mathbb{R}^d$ are \textit{random vectors} acting as class centroids for object and temporal information. The ideal decomposition satisfies $I(\mu^c;\, \mu^y) = 0$: the two components are mutually uninformative, and their sum captures the full label-relevant content of $v$.

A proxy $p \in \mathcal{P}$ is a learnable vector of the same dimensionality as the embedding space, representing the centroid of a contrastive class~\cite{movshovitz2017no}. Unlike dataset-derived centroids, proxies are optimized end-to-end: they are simultaneously pulled toward the embeddings of their assigned class and pushed away from all others, converging to dynamic but class-specific attractors in $\mathbb{R}^d$. We define three proxy sets:

\paragraph{Category proxies $\mathcal{P}^c = \{\mu^c_1,\dots,\mu^c_C\}$.} One proxy per object category. Each $\mu^c$ integrates over all years of category $c$, converging to a year-agnostic centroid. These operate on \textit{unnormalized} embeddings, so both magnitude and direction contribute to proxy assignment.

\paragraph{Year proxies $\mathcal{P}^y = \{\mu^y_1,\dots,\mu^y_Y\}$.} One proxy per year, shared across all categories. These operate on $\ell_2$-normalized residuals $\hat{u} = u / \|u\|$, reducing the dot product to cosine similarity and encoding temporal information purely in the \textit{angle} of the representation — a design choice validated by our ablation study (Table~\ref{tab:ablation}).

\paragraph{Auxiliary displacement proxies $\mathcal{K}^y = \{k^y_1,\dots,k^y_Y\}$.} One learnable displacement vector per year, shared across categories. These provide learned temporal offsets used during training to enforce transitivity (Definition~\ref{def:temporal_trans}). The set $\mathcal{K}^y$ is discarded at inference.

Under the model assumptions of Definition~\ref{def:definition_classes_assumption}, the following decomposability result holds:

\begin{proposition}[Temporally Decomposable Image Representation]
\label{prop:tdir}
Under the model assumptions of Definition~\ref{def:definition_classes_assumption}:
\begin{equation}
I(\mu^c, \mu^y;\, c,y) = I(\mu^c;\, c) + I(\mu^y;\, y)
\end{equation}
\end{proposition}

See the proof in Appendix~\ref{sec:tfir_demonstraction}. This implies that, under the class independence assumption, a visual embedding $v = f(x)$ can be \textbf{temporally decomposable} with respect to the proxy sets $\mathcal{P}^c$ and $\mathcal{P}^y$. Because $v = \mu^y + \mu^c$, the mutual information between the joint labels and the full representation decomposes as:
\begin{equation}
I(v;\, c,y) = I(\mu^c;\, c) + I(\mu^y;\, y)
\end{equation}

Ideally, the total information about both factors contained in $v$ is fully decomposable into independent contributions: categorical information $I(\mu^c;\, c)$ and temporal information $I(\mu^y;\, y)$. In the ideal case, $\mu^c$ and $\mu^y$ are full descriptors of $v$ with respect to $(c, y)$: knowing both proxies captures everything $v$ encodes about category and year, with no label-relevant information remaining.

In practice, however, there is some statistical entanglement between the appearance of certain objects and their respective dates. We therefore account for a residual signal $\varepsilon \in \mathbb{R}^d$ present in $v$:
\begin{equation}
v = \mu^c + \mu^y + \epsilon
\end{equation}

\begin{proposition}[Practical Learnability of TDIRs]
\label{prop:learnability}
Under the TDIR framework, the discrepancy term $\delta \triangleq I(\varepsilon;\, c, y)$ depends exclusively on the statistical structure of the training data and not on the choice of proxy vectors $\mu^c$ or $\mu^y$:
\begin{equation} 
I(v;\, c,y)
\leq
I(\mu^c;\, c)
+
I(\mu^y;\, y)
+
\delta \qquad \text{where} \qquad \delta
\triangleq
I(\varepsilon;\, c,y)
\end{equation}

\end{proposition}

Consequently, $\delta$ constitutes a data-dependent constant with respect to the optimization, and the category and year centroids can be learned freely to satisfy the orthogonality constraint $\langle \mu^c,\, \mu^y \rangle = 0$ without affecting this error bound (see proof in Appendix~\ref{sec:tfir_demonstraction_error}). A representation $v$ becomes more temporally decomposable as $\delta$ decreases. This leaves open a potential circularity: Proposition~\ref{prop:learnability} assumes orthogonality to guarantee free learning of $\mu^c$ and $\mu^y$, yet this orthogonality has not been enforced. Theorem~\ref{theorem:emergent_manuscript} will resolve this by showing that orthogonality emerges implicitly from the proposed training strategy.

\subsection{Residual Vectors}
\label{sec:residuals}

Given the proxy sets above, we build four residual vectors from each training tuple $T = (x_A, x_B)$ with shared category $c$ and distinct years $y_A \neq y_B$. The four-step geometric construction is illustrated in Figure~\ref{fig:optimization}.

\paragraph{Direct residuals.} Subtracting the category centroid $\mu^c$ re-centers each cluster at the origin, revealing the temporal component (Figure~\ref{fig:optimization}b):
\begin{equation}
    r_A = v_A - \mu^c, \qquad r_B = v_B - \mu^c.
\end{equation}

In $r_A$ and $r_B$, the information necessary to infer the category has been removed. Consequently, even if temporal information were present in the category subspace, it is stripped from the representation and will not propagate to subsequent steps.

\paragraph{Shifted residuals (swapping trick).} To enforce temporal transitivity, we additionally form two cross-year residuals using the auxiliary displacement proxies $\mathcal{K}^y$:
\begin{equation}
r_{A \to B} = (v_A - \mu^c) + k^y_B, \qquad r_{B \to A} = (v_B - \mu^c) + k^y_A.
\end{equation}
The vector $r_{A \to B}$ displaces image $A$'s temporal residual toward year $y_B$; if the temporal manifold is truly transitive, this shifted residual should be indistinguishable from $r_B$ — and therefore close to $\mu^{y_B}$. The symmetric argument holds for $r_{B \to A}$. %
All four residuals are $\ell_2$-normalized before being passed to the year loss:
\begin{equation}
\hat{r}_* \gets r_* / \|r_*\|.
\end{equation}
This normalization encodes temporal information purely as angular structure in the shared subspace (Figure~\ref{fig:optimization}c--d).

\begin{figure}
    \centering
    \setlength{\tabcolsep}{4pt}
    \begin{tabular}{cccc}
        \parbox{0.23\textwidth}{
            \centering
            \resizebox{\linewidth}{!}{
\definecolor{c0}{RGB}{120,15,50}
\definecolor{c1}{RGB}{185,45,35}
\definecolor{c2}{RGB}{220,105,30}
\definecolor{c3}{RGB}{200,170,60}
\definecolor{c4}{RGB}{100,165,100}
\definecolor{c5}{RGB}{45,140,165}
\definecolor{c6}{RGB}{30,90,175}
\definecolor{c7}{RGB}{140,170,210}

\begin{tikzpicture}
\node[font=\bfseries\Huge, black] at (0,5) {Category Clustering};

\draw[gray!15, line width=0.30pt, step=0.5] (-5.5,-4.5) grid (5.5,4.5);
\draw[gray!30, line width=0.45pt, step=1.0] (-5.5,-4.5) grid (5.5,4.5);





\foreach \c [count=\j from 0] in {c0, c1, c2, c3, c4, c5, c6, c7} {
    \foreach \i in {1,...,3} {
        
        
        \pgfmathsetmacro{\ang}{rnd*360}
        \pgfmathsetmacro{\rad}{0.25 + rnd*0.8}
        \pgfmathsetmacro{\x}{\rad*cos(\ang)}
        \pgfmathsetmacro{\y}{\rad*sin(\ang)}
        \filldraw[\c, draw=\c!70!black, line width=0.28pt, fill opacity=0.93] 
          (\x -2.8,\y+2.4) circle (0.2);

    }
}

\newcommand{\tri}[4]{%
  \filldraw[#4, draw=#4!65!black, line width=0.28pt, fill opacity=0.93]
    (#1,{#2+#3}) -- ({#1-0.866*#3},{#2-0.5*#3}) -- ({#1+0.866*#3},{#2-0.5*#3}) -- cycle;}

\foreach \c [count=\j from 0] in {c0, c1, c2, c3, c4, c5, c6, c7} {
    \foreach \i in {1,...,4} {

        \pgfmathsetmacro{\ang}{rnd*360}
        \pgfmathsetmacro{\rad}{0.25 + rnd*0.9}
        \pgfmathsetmacro{\x}{\rad*cos(\ang)}
        \pgfmathsetmacro{\y}{\rad*sin(\ang)}
        \tri{\x + 3.14}{\y + 2.80}{0.25}{\c}

    }
}

\newcommand{\sq}[4]{%
  \filldraw[#4, draw=#4!65!black, line width=0.28pt, fill opacity=0.93]
    ({#1-#3},{#2-#3}) rectangle ({#1+#3},{#2+#3});}


\foreach \c [count=\j from 0] in {c0, c1, c2, c3, c4, c5, c6, c7} {
    \foreach \i in {1,...,5} {

        \pgfmathsetmacro{\ang}{rnd*360}
        \pgfmathsetmacro{\rad}{0.25 + rnd*0.9}
        \pgfmathsetmacro{\x}{\rad*cos(\ang)}
        \pgfmathsetmacro{\y}{\rad*sin(\ang)}
        \sq{\x + 0.09}{\y - 3.3}{0.2}{\c}
    }
}

\end{tikzpicture}}\\
            \small (a)
        }
        &
        \parbox{0.23\textwidth}{
            \centering
            \resizebox{\linewidth}{!}{
\definecolor{c0}{RGB}{120,15,50}
\definecolor{c1}{RGB}{185,45,35}
\definecolor{c2}{RGB}{220,105,30}
\definecolor{c3}{RGB}{200,170,60}
\definecolor{c4}{RGB}{100,165,100}
\definecolor{c5}{RGB}{45,140,165}
\definecolor{c6}{RGB}{30,90,175}
\definecolor{c7}{RGB}{140,170,210}

\begin{tikzpicture}
\node[font=\bfseries\Huge, black] at (0,5) {Centering};

\draw[gray!15, line width=0.30pt, step=0.5] (-5.5,-4.5) grid (5.5,4.5);
\draw[gray!30, line width=0.45pt, step=1.0] (-5.5,-4.5) grid (5.5,4.5);



\newcommand{\arrowScale}{0.8} 
\newcommand{\ScaledArrow}[3]{%
  \draw[-{Stealth[length=6pt,width=4pt]},
        gray, line width=0.9pt]
     ({#1*\arrowScale},{#2*\arrowScale}) -- (0, 0)
    node[pos=-0.2, sloped, above, font=\tiny\itshape, text=black]
    {\textbf{\Huge{$#3$}}};
}

\ScaledArrow{-3.0}{2.5}{\textbf{- $\text{Proxy}_\text{A}$}}
\ScaledArrow{3.0}{2.5}{\textbf{- $\text{Proxy}_\text{B}$}}
\ScaledArrow{0.2}{-3.2}{\textbf{- $\text{Proxy}_\text{C}$}}

\newcommand{\tri}[4]{%
  \filldraw[#4, draw=#4!65!black, line width=0.28pt, fill opacity=0.93]
    (#1,{#2+#3}) -- ({#1-0.866*#3},{#2-0.5*#3}) -- ({#1+0.866*#3},{#2-0.5*#3}) -- cycle;}

\newcommand{\sq}[4]{%
  \filldraw[#4, draw=#4!65!black, line width=0.28pt, fill opacity=0.93]
    ({#1-#3},{#2-#3}) rectangle ({#1+#3},{#2+#3});}

\foreach \c [count=\j from 0] in {c0, c1, c2, c3, c4, c5, c6, c7} {
    \foreach \i in {1,...,3} {

        \pgfmathsetmacro{\ang}{rnd*360}
        \pgfmathsetmacro{\rad}{sqrt(rnd)*1.3} 
        \pgfmathsetmacro{\x}{\rad*cos(\ang)}
        \pgfmathsetmacro{\y}{\rad*sin(\ang)}
        \sq{\x}{\y }{0.2}{\c}

        \pgfmathsetmacro{\ang}{rnd*360}
        \pgfmathsetmacro{\rad}{sqrt(rnd)*1.3} 
        \pgfmathsetmacro{\x}{\rad*cos(\ang)}
        \pgfmathsetmacro{\y}{\rad*sin(\ang)}
        \tri{\x }{\y }{0.3}{\c}

        \pgfmathsetmacro{\ang}{rnd*360}
        \pgfmathsetmacro{\rad}{sqrt(rnd)*1.3} 
        \pgfmathsetmacro{\x}{\rad*cos(\ang)}
        \pgfmathsetmacro{\y}{\rad*sin(\ang)}
        \filldraw[\c, draw=\c!70!black, line width=0.28pt, fill opacity=0.93] 
          (\x ,\y) circle (0.2);
    }
}

\end{tikzpicture}}\\
            \small (b)
        }
        &
        \parbox{0.23\textwidth}{
            \centering
            \resizebox{\linewidth}{!}{
\definecolor{c0}{RGB}{120,15,50}
\definecolor{c1}{RGB}{185,45,35}
\definecolor{c2}{RGB}{220,105,30}
\definecolor{c3}{RGB}{200,170,60}
\definecolor{c4}{RGB}{100,165,100}
\definecolor{c5}{RGB}{45,140,165}
\definecolor{c6}{RGB}{30,90,175}
\definecolor{c7}{RGB}{140,170,210}

\begin{tikzpicture}
\node[font=\bfseries\Huge, black] at (0,5)  {Date Proxy Optimization};
\draw[gray!15, line width=0.30pt, step=0.5] (-5.5,-4.5) grid (5.5,4.5);
\draw[gray!30, line width=0.45pt, step=1.0] (-5.5,-4.5) grid (5.5,4.5);

\def\arcRad{3}

\draw[-{Stealth[length=7pt,width=5pt]}, line width=1.2pt, gray!80] 
  (165:\arcRad) arc (165:45:\arcRad);

\draw[
    ultra thick,
    -{Stealth[length=16pt,width=12pt]},
    decorate
] (165:\arcRad) arc (165:45:\arcRad);

\newcommand{\tri}[4]{%
  \filldraw[#4, draw=#4!65!black, line width=0.28pt, fill opacity=0.93]
    (#1,{#2+#3}) -- ({#1-0.866*#3},{#2-0.5*#3}) -- ({#1+0.866*#3},{#2-0.5*#3}) -- cycle;}

\newcommand{\sq}[4]{%
  \filldraw[#4, draw=#4!65!black, line width=0.28pt, fill opacity=0.93]
    ({#1-#3},{#2-#3}) rectangle ({#1+#3},{#2+#3});}

\pgfmathsetseed{42}

\foreach \c [count=\j from 0] in {c0, c1, c2, c3, c4, c5, c6, c7} {
    \foreach \i in {1,...,7} {
        
        
        \pgfmathsetmacro{\ang}{135 - \j*45 + rand*22}
        \pgfmathsetmacro{\rad}{0.25 + rnd*1.3}
        \pgfmathsetmacro{\x}{\rad*cos(\ang)}
        \pgfmathsetmacro{\y}{\rad*sin(\ang)}
        \filldraw[\c, draw=\c!70!black, line width=0.28pt, fill opacity=0.93] 
          (\x,\y) circle (0.135);
        
        \pgfmathsetmacro{\ang}{135 - \j*45 + rand*22}
        \pgfmathsetmacro{\rad}{0.25 + rnd*1.3}
        \pgfmathsetmacro{\x}{\rad*cos(\ang)}
        \pgfmathsetmacro{\y}{\rad*sin(\ang)}
        \tri{\x}{\y}{0.132}{\c}
        
        \pgfmathsetmacro{\ang}{135 - \j*45 + rand*22}
        \pgfmathsetmacro{\rad}{0.25 + rnd*1.3}
        \pgfmathsetmacro{\x}{\rad*cos(\ang)}
        \pgfmathsetmacro{\y}{\rad*sin(\ang)}
        \sq{\x}{\y}{0.118}{\c}
    }
}

\end{tikzpicture}}\\
            \small (c)
        }
        &
        \parbox{0.23\textwidth}{
            \centering
            \resizebox{\linewidth}{!}{
\definecolor{c0}{RGB}{120,15,50}
\definecolor{c1}{RGB}{185,45,35}
\definecolor{c2}{RGB}{220,105,30}
\definecolor{c3}{RGB}{200,170,60}
\definecolor{c4}{RGB}{100,165,100}
\definecolor{c5}{RGB}{45,140,165}
\definecolor{c6}{RGB}{30,90,175}
\definecolor{c7}{RGB}{140,170,210}

\begin{tikzpicture}
\node[font=\bfseries\Huge, black] at (0,5) {Category Translation};

\draw[gray!15, line width=0.30pt, step=0.5] (-5.5,-4.5) grid (5.5,4.5);
\draw[gray!30, line width=0.45pt, step=1.0] (-5.5,-4.5) grid (5.5,4.5);



\newcommand{\arrowScale}{0.8} 
\newcommand{\ScaledArrow}[3]{%
  \draw[-{Stealth[length=6pt,width=4pt]},
        gray, line width=0.9pt]
    (0,0) -- ({#1*\arrowScale},{#2*\arrowScale})
    node[pos=0.5, sloped, above, font=\tiny\itshape, text=black]
    {\textbf{\Huge{$#3$}}};
}

\ScaledArrow{-3.0}{2.5}{+\text{P}_A}
\ScaledArrow{3.0}{2.5}{+\text{P}_B}
\ScaledArrow{0.2}{-3.2}{+\text{P}_C}

\foreach \c [count=\j from 0] in {c0, c1, c2, c3, c4, c5, c6, c7} {
    \foreach \i in {1,...,4.5} {
        
        
        \pgfmathsetmacro{\ang}{135 - \j*45 + rand*22}
        \pgfmathsetmacro{\rad}{0.25 + rnd*1}
        \pgfmathsetmacro{\x}{\rad*cos(\ang)}
        \pgfmathsetmacro{\y}{\rad*sin(\ang)}
        \filldraw[\c, draw=\c!70!black, line width=0.28pt, fill opacity=0.93] 
          (\x -2.8,\y+2.4) circle (0.2);

    }
}

\newcommand{\tri}[4]{%
  \filldraw[#4, draw=#4!65!black, line width=0.28pt, fill opacity=0.93]
    (#1,{#2+#3}) -- ({#1-0.866*#3},{#2-0.5*#3}) -- ({#1+0.866*#3},{#2-0.5*#3}) -- cycle;}

\foreach \c [count=\j from 0] in {c0, c1, c2, c3, c4, c5, c6, c7} {
    \foreach \i in {1,...,4} {

        \pgfmathsetmacro{\ang}{135 - \j*45 + rand*22}
        \pgfmathsetmacro{\rad}{0.25 + rnd*0.9}
        \pgfmathsetmacro{\x}{\rad*cos(\ang)}
        \pgfmathsetmacro{\y}{\rad*sin(\ang)}
        \tri{\x + 3.14}{\y + 2.80}{0.3}{\c}

    }
}

\newcommand{\sq}[4]{%
  \filldraw[#4, draw=#4!65!black, line width=0.28pt, fill opacity=0.93]
    ({#1-#3},{#2-#3}) rectangle ({#1+#3},{#2+#3});}


\foreach \c [count=\j from 0] in {c0, c1, c2, c3, c4, c5, c6, c7} {
    \foreach \i in {1,...,5} {

        \pgfmathsetmacro{\ang}{135 - \j*45 + rand*22}
        \pgfmathsetmacro{\rad}{0.25 + rnd*0.9}
        \pgfmathsetmacro{\x}{\rad*cos(\ang)}
        \pgfmathsetmacro{\y}{\rad*sin(\ang)}
        \sq{\x + 0.09}{\y - 3.4}{0.2}{\c}
    }
}

\end{tikzpicture}}\\
            \small (d)
        }
    \end{tabular}
    \vspace{5pt}
    \caption{Training pipeline. (a) Proxy loss forms year-agnostic category clusters.
    (b) Subtracting $\mu^c$ centers each cluster into a shared temporal subspace.
    (c) Year proxies are optimized on normalized residuals via cosine distance.
    (d) The swapping trick translates residuals across years, enforcing temporal transitivity.}
    \label{fig:optimization}
\end{figure}
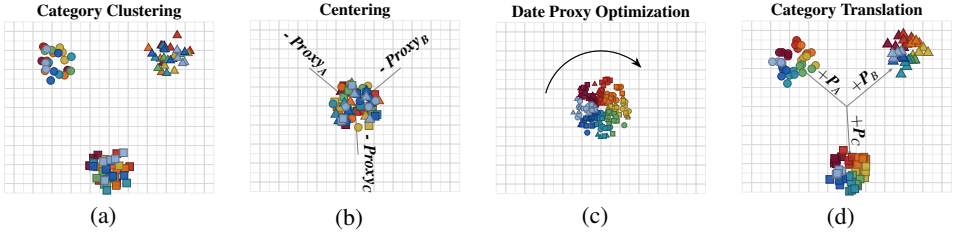

\subsection{Proxy Losses and Joint Objective}
\label{sec:losses}

\paragraph{Proxy loss.} All objectives share a single form~\cite{movshovitz2017no}. For an embedding $u$ with ground-truth proxy $\mu^* \in \mathcal{P}$, the proxy loss is:
\begin{equation}
    \mathcal{L}(u,\, \mu^*;\, \mathcal{P}) = -\log \frac{\exp(u^\top \mu^*)}
{\displaystyle\sum_{p \in \mathcal{P}} \exp(u^\top p)}.
\end{equation}

This is a softmax cross-entropy over proxy assignments: it maximizes the score of the correct proxy $\mu^*$ relative to all others in $\mathcal{P}$, pulling $u$ into the neighborhood of $\mu^*$ while repelling it from every competing proxy.

\paragraph{Category loss.} The category subspace is optimized directly on the full (unnormalized) embeddings:
\begin{equation}
\label{eq:catloss}
\mathcal{L}^c(T) = \mathcal{L}(v_A,\, \mu^c;\, \mathcal{P}^c) \;+\; \mathcal{L}(v_B,\, \mu^c;\, \mathcal{P}^c).
\end{equation}
Each proxy $\mu^c$ integrates over all years of category $c$ (Figure~\ref{fig:optimization}a). At this stage, date information may still be present in the embedding through spurious correlations such as textures or color; the subsequent residual construction removes it.

\paragraph{Direct temporal loss.} The year subspace is optimized on the normalized direct residuals:
\begin{equation}
\label{eq:direct_year_loss}
\mathcal{L}^y_\text{direct}(T) = \mathcal{L}(\hat{r}_A,\, \mu^{y_A};\, \mathcal{P}^y) \;+\;
\mathcal{L}(\hat{r}_B,\, \mu^{y_B};\, \mathcal{P}^y).
\end{equation}
This makes the year subspace \textit{estimable}, but not yet geometrically consistent across categories.

\paragraph{Transitive temporal loss.} To enforce temporal transitivity (Definition~\ref{def:temporal_trans}), the shifted residuals are pulled toward their \textit{target} year proxies — the year each residual has been displaced toward, rather than the image's own year:
\begin{equation}
\mathcal{L}^y_\text{trans}(T) =
\mathcal{L}(\hat{r}_{A \to B},\, \mu^{y_B};\, \mathcal{P}^y)
\;+\;
\mathcal{L}(\hat{r}_{B \to A},\, \mu^{y_A};\, \mathcal{P}^y).
\end{equation}
This cannot be satisfied unless the displacement $k^y$ is consistent across all categories — that is, unless the year residuals of any two images from any two categories differ by the same offset in the temporal subspace. This \textit{swapping trick} directly enforces Definition~\ref{def:temporal_trans}.

\paragraph{Total year loss and joint objective.} Combining direct and transitive terms (Figure~\ref{fig:optimization}d):
\begin{equation}
\mathcal{L}^y(T) =
\underbrace{\mathcal{L}(\hat{r}_A,\, \mu^{y_A};\, \mathcal{P}^y)
+ \mathcal{L}(\hat{r}_B,\, \mu^{y_B};\, \mathcal{P}^y)}_{\text{direct: } r \approx \mu^y}
+
\underbrace{\mathcal{L}(\hat{r}_{A \to B},\, \mu^{y_B};\, \mathcal{P}^y)
+ \mathcal{L}(\hat{r}_{B \to A},\, \mu^{y_A};\, \mathcal{P}^y)}_{\text{transitive: } r + k^y \approx \mu^y}.
\label{eq:lyear}
\end{equation}

The total minimized loss is:
\begin{equation}
\label{eq:total_loss}
    \mathcal{L} = \mathcal{L}^c(T) + \mathcal{L}^y(T).
\end{equation}

\begin{theorem}[Emergent TDIR]
\label{theorem:emergent_manuscript}
Let $\mathcal{L}^c$ and $\mathcal{L}^y$ be the proxy loss functions defined over an embedding space $\mathbb{R}^d$. Then:

\textbf{Part I (exact case).} If $I(\epsilon\,;\,c,y) = 0$, joint minimization of 
$\mathcal{L} = \mathcal{L}^c + \mathcal{L}^y$ implies TDIR:
\begin{equation}
    \min_\theta\,\mathcal{L}^c + \mathcal{L}^y \implies 
    p(\mu^c, \mu^y, c, y) = p(\mu^c, c)\cdot p(\mu^y, y)
\end{equation}

\textbf{Part II (asymptotic case).} If $I(\epsilon\,;\,c,y) \neq 0$, joint 
minimization implies asymptotic TDIR:
\begin{equation}
    \min_\theta\,\mathcal{L}^c + \mathcal{L}^y \implies 
    \langle\mu^c,\,\mu^y\rangle \to O\!\left(\frac{I(\epsilon\,;\,c,y)}{\|\mu^y\|^2}
    \right) \quad \forall\,c,\,y
\end{equation}
\end{theorem}

See proof in Appendix~\ref{sec:theorem_emergent}. Note that Proposition~\ref{prop:learnability} states that $\mu^y$ and $\mu^c$ can be learned under the assumption that temporal and categorical centroids are orthogonal. Although this might seem a strong constraint, Theorem~\ref{theorem:emergent_manuscript} shows that the joint optimization of Eq.~\eqref{eq:total_loss} imposes this orthogonality as an \textit{emergent} property of the centering and swapping tricks. It is therefore formally guaranteed that Algorithm~\ref{alg:training} can yield TDIR in both the ideal and practical case, through an implicit orthogonality regularization arising from the swapping of temporal variables across common object categories.

\begin{algorithm}[t]
\caption{TDIR Training Loop}
\label{alg:training}
\begin{algorithmic}[1]
\setlength{\itemsep}{2pt}
\State \textbf{Initialize:} $\mathcal{P}^c \in \mathbb{R}^{C \times d}$, $\mathcal{P}^y \in \mathbb{R}^{Y \times d}$, $\mathcal{K} \in \mathbb{R}^{Y \times d}$

\For{epoch $e = 1$ \ldots $E$}
    \For{$(x_A, x_B, c) \in \mathcal{D}_{\text{train}}$}

    \State \fboxsep=3pt\colorbox{blue!5}{\parbox{\dimexpr\linewidth-4em}{$v_A \gets f(x_A)$; \quad $v_B \gets f(x_B)$ \hfill \small\textit{forward pass}}}

    \State \fboxsep=3pt\colorbox{red!5}{\parbox{\dimexpr\linewidth-4em}{$\mathcal{L}^c \gets \mathcal{L}(v_A, c; \mathcal{P}^c) + \mathcal{L}(v_B, c; \mathcal{P}^c)$ \hfill \small\textit{category subspace}}}

    \State \fboxsep=3pt\colorbox{green!5}{\parbox{\dimexpr\linewidth-4em}{$r_A = v_A - \mu^c_c$; \quad $r_B = v_B - \mu^c_c$ \hfill \small\textit{center by category (decomposable)}}}

    \State \fboxsep=3pt\colorbox{green!5}{\parbox{\dimexpr\linewidth-4em}{$r_{A \to B} = v_A + k_{B} - \mu^c_c$; \quad $r_{B \to A} = v_B + k_A - \mu^c_c$ \hfill \small\textit{swapping trick (transitive)}}}

    \State \fboxsep=3pt\colorbox{orange!5}{\parbox{\dimexpr\linewidth-4em}{$\hat{r}_* \gets r_* / \|r_*\|$ \hfill \small\textit{normalize for angular encoding}}}

\State \fboxsep=3pt\colorbox{red!5}{\parbox{\dimexpr\linewidth-4em}{%
    $\mathcal{L}^y = \mathcal{L}(\hat{r}_A, \mu^{y_A}; \mathcal{P}^y) + \mathcal{L}(\hat{r}_B, \mu^{y_B}; \mathcal{P}^y)\hspace*{0.2em} +$
    \hfill \small\textit{estimable temporal embedding}\\[0.8em]
    \hspace*{0.2em}\;\;
    $+\hspace*{0.2em} \mathcal{L}(\hat{r}_{A \to B}, \mu^{y_B}; \mathcal{P}^y) + \mathcal{L}(\hat{r}_{B \to A}, \mu^{y_A}; \mathcal{P}^y)$
    \hfill \small\textit{transitive temporal embedding}%
}}

\State \fboxsep=3pt\colorbox{blue!5}{\parbox{\dimexpr\linewidth-4em}{%
$\{f,\mathcal{P}^c,\mathcal{P}^y,\mathcal{K}\}
\leftarrow
\{f,\mathcal{P}^c,\mathcal{P}^y,\mathcal{K}\}
- \alpha \nabla(\mathcal{L}^c + \mathcal{L}^y)$
\hfill \small\textit{backward pass}%
}}

    \EndFor
\EndFor
\end{algorithmic}
\end{algorithm}

\subsection{Inference and Composed Historical Image Retrieval}
\label{sec:inference}

At inference $\mathcal{K}^y$ is discarded. The three retrieval modes (Figure~\ref{fig:inference}) 
follow directly from the decomposition $f(x) = v \approx \mu^c + \mu^y$.

\paragraph{Label-based (Figure~\ref{fig:inference}a).}

At inference, we construct a database of test embeddings $\{f(x_i)\}_{i=1}^N$. As in any common embedding-based retrieval, we 
construct a query vector $q \in \mathbb{R}^d$ and returning its nearest neighbors:
\begin{equation}
    \{x_{(1)}, x_{(2)}, \dots, x_{(k)}\} = \underset{x_i}{\text{top-}k}\;
\frac{f(x_i)^\top q}{\|f(x_i)\|\|q\|}.
\end{equation}

By the decomposability property, a query targeting category $c$ at year $y$ is simply the sum 
of their respective proxies:
\begin{equation}
q = \mu^c + \mu^y.
\end{equation}

Since $\langle \mu^c, \mu^y \rangle = 0$, the query lies precisely at the intersection of both 
subspaces, and the retrieved images $\{x_{(i)}\}$ are those whose embeddings are simultaneously 
close to $\mu^c$ \textit{and} $\mu^y$ — that is, images of category $c$ from year $y$.
\paragraph{Image + label (Figure~\ref{fig:inference}b).}

When the target year $y$ is known but no category label is provided, we replace $\mu^c$ with the 
actual image embedding:
\begin{equation}
    q = f(x_c) + \mu^y.
\end{equation}

This is strictly more expressive than the label-based mode: rather than retrieving images close 
to the category centroid $\mu^c$, the query anchors to the instance-specific content of $x_c$ — 
including the $\varepsilon$ residual — while the year direction is fully determined by $\mu^y$. 
Retrieved images thus share the particular visual characteristics of $x_c$, translated to year $y$. 

\paragraph{Image + image, zero-label (Figure~\ref{fig:inference}c)}

The most powerful inference mode requires no label information whatsoever. The user provides two 
images: a \textit{category image} $x_c$ — whose year is entirely unknown and irrelevant — and a 
\textit{temporal reference} $x_y$, from which we wish to borrow the year.

\textbf{Step 1: Infer the category of $x_y$.} Since no category label is provided, we assign 
$x_y$ to its nearest category proxy:
\begin{equation}
    \hat{c} = \underset{c'}{\arg\min}\;\|f(x_y) - \mu^{c'}\|_2.
\end{equation}

\textbf{Step 2: Extract the temporal information} Subtracting the inferred category proxy strips 
the category information from $f(x_y)$, leaving a category-agnostic vector:
\begin{equation}
r_y = f(x_y) - \mu^{\hat{c}} \;\approx\; \mu^{y} + \varepsilon,
\end{equation}

where $\mu^y$ is the (unknown to the user) year proxy and $\varepsilon$ is the instance-specific 
residual. Crucially, neither $\mu^y$ nor the year of $x_y$ need to be known — the vector $r_y$ 
\textit{\textbf{is}} the temporal information.

\textbf{Step 3: Temporal transplant.} We add $r_y$ to the category embedding of $x_c$:
\begin{equation}
    q = f(x_c) + r_y.
\end{equation}

By temporal transitivity (Definition~\ref{def:temporal_trans}), this displaces $f(x_c)$ along the temporal manifold 
toward the year of $x_y$, while preserving its category direction. The $k$-nearest neighbors of 
$q$ are thus images of the same category as $x_c$, at the same period as $x_y$.
\begin{figure}
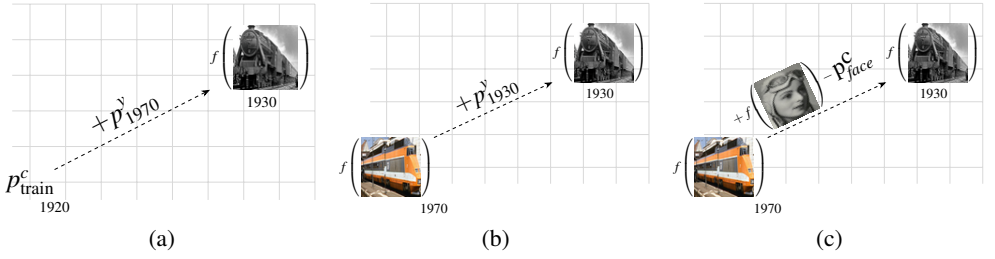

    \centering
    \setlength{\tabcolsep}{4pt}
    \begin{tabular}{ccc}
        \parbox{0.32\textwidth}{
            \centering
            \resizebox{\linewidth}{!}{\input{images/inference/inference_1}}\\
            \small (a)
        }
        &
        \parbox{0.32\textwidth}{
            \centering
            \resizebox{\linewidth}{!}{\input{images/inference/inference_2}}\\
            \small (b)
        }
        &
        \parbox{0.32\textwidth}{
            \centering
            \resizebox{\linewidth}{!}{\input{images/inference/inference_3}}\\
            \small (c)
        }
    \end{tabular}
    \vspace{5pt}
    \caption{Inference modes. (a) Label-based retrieval via proxy addition. 
    (b) Image-guided retrieval with a target year label. 
    (c) Zero-label retrieval: the year residual of $x_y$ is extracted and transplanted onto $f(x_c)$.}
    \label{fig:inference}
\end{figure}

\section{Experimental Set-Up}
\subsection{Dataset}
For the correct application and benchmarking of our proposed problem set-up, it is for us required to utilize a dataset where both the date and object annotations are available. For doing so, we take advantage of the Date Estimation in The Wild dataset (DEW) \cite{muller2017picture} with object-specific detections by using the same criteria as in \cite{net2024transformer}. In this Section, we detail the dataset and how it differs from the original DEW images.

\textbf{DEW Dataset}
The Date Estimation in The Wild dataset contains 1M natural-scene photographs from 1930 to 1999, the date information has an excellent granularity of a year per photo. However, the images are natural scenes and, therefore, contain a variety of different objects which hinders the applicability of object-specific representations.

\textbf{Object-Centric DEW Dataset}
In \cite{net2024transformer}, the authors propose a detection pipeline to crop objects detected with a minimum 10000px resolution by using the DETR \cite{carion2020end} model. Although the object-level detection (bounding boxes) are publicly available, we note that the authors use a constrained set of 6 categories, which in our case could trivialize the object-sensitive subspace.

\textbf{Proposed Dataset}
To solve the aforementioned issue, we use the same resolution criteria, but utilize OWLv2 \cite{minderer2023scaling} to detect objects from all possible categories in the DETR \texttt{HuggingFace} implementation \cite{huggingface_detr_docs} with a 30\% detection threshold\footnote{This is a standard threshold when using OWL \cite{huggingface_owlvit_docs}.}. This leads us to 7,239,083 object detections of medium-to-high resolution of 586 objects and date-specific annotations from DEW (see Table~\ref{tab:proposed_dataset}). In this article we will focus on the sub-set of top-50 most predominant objects and 4,478,887 detections (see Appendix \ref{sec:objectAnalysis}). But the complete set of detections, which can be utilized to expand the presented method or for many other applications, is available to download\footnote{Contact \href{mailto:amolina@cvc.uab.cat}{amolina@cvc.uab.cat} with the string ``\texttt{[TDIR] - Access to the database}'' as subject.}. The test partition follows the same image selection as the original DEW dataset, from which we separate the detections according to the image source.

\begin{table}[t]
    \centering
    \small
    \setlength{\tabcolsep}{6pt}
    \renewcommand{\arraystretch}{1.6}

    \begin{tabular}{lccccccc}
        \toprule
         & \textbf{1930s} & \textbf{1940s} & \textbf{1950s} & \textbf{1960s} & \textbf{1970s} & \textbf{1980s} & \textbf{1990s} \\
        \midrule
        
        Jacket &
        \includegraphics[width=1cm]{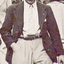} &
        \includegraphics[width=1cm]{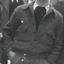} &
        \includegraphics[width=1cm]{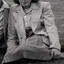} &
        \includegraphics[width=1cm]{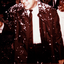} &
        \includegraphics[width=1cm]{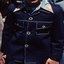} &
        \includegraphics[width=1cm]{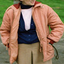} &
        \includegraphics[width=1cm]{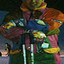} \\
        
        Poster &
        \includegraphics[width=1cm]{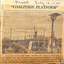} &
        \includegraphics[width=1cm]{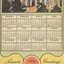} &
        \includegraphics[width=1cm]{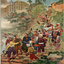} &
        \includegraphics[width=1cm]{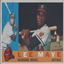} &
        \includegraphics[width=1cm]{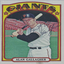} &
        \includegraphics[width=1cm]{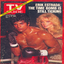} &
        \includegraphics[width=1cm]{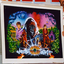} \\
        
        Car &
        \includegraphics[width=1cm]{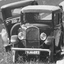} &
        \includegraphics[width=1cm]{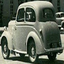} &
        \includegraphics[width=1cm]{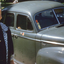} &
        \includegraphics[width=1cm]{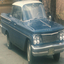} &
        \includegraphics[width=1cm]{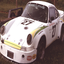} &
        \includegraphics[width=1cm]{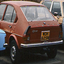} &
        \includegraphics[width=1cm]{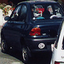} \\

        Building &
        \includegraphics[width=1cm]{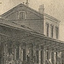} &
        \includegraphics[width=1cm]{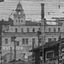} &
        \includegraphics[width=1cm]{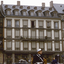} &
        \includegraphics[width=1cm]{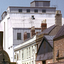} &
        \includegraphics[width=1cm]{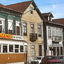} &
        \includegraphics[width=1cm]{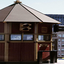} &
        \includegraphics[width=1cm]{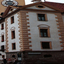} \\
        
        \bottomrule
    \end{tabular}

    \caption{Example images and categories from the dataset.}
    \label{tab:proposed_dataset}
\vspace{-2em}
\end{table}

\subsection{Implementation Details}

The practical implementation of the Temporally Decomposable Image Representations has been evaluated using several backbone architectures in order to assess the robustness of the proposed methodology across both convolutional and transformer-based visual representations. In particular, experiments have been conducted with \textbf{ConvNeXt-Base} \cite{liu2022convnet}, \textbf{Vision Transformer B/32 (ViT-B/32)} \cite{dosovitskiy2020image}, \textbf{ResNet50} \cite{he2016deep}, and \textbf{VGG19-BN} \cite{simonyan2014very}, all initialized with \textit{ImageNet} pre-trained weights \cite{deng2009imagenet} using the official \textit{PyTorch} implementations \cite{torchvision_convnext_base}. For all architectures, the original classification head was removed and replaced with a projection layer producing a fixed embedding dimension of 1024. This unified embedding size ensures a fair comparison between architectures and allows all models to operate under the same metric-learning framework. Unlike previous approaches proposed for solving the DEW dataset \cite{muller2017picture} 
 in an object-centric manner \cite{net2024transformer}, which requires ensembles of specialized models for each semantic category, the proposed methodology employs a single unified backbone capable of handling all categories jointly. This significantly scales well with the number of object categories.
All backbone architectures were trained using the same optimization setup to ensure experimental consistency. The AdamW optimizer jointly optimizes the backbone parameters, the learnable category proxies, and the auxiliary temporal embeddings ($f$, $\mu^*$, and $k^y$) with a learning rate of $10^{-4}$ during 10 epochs, using the \texttt{ProxyNCALoss} proposed by Movshovitz-Attias et al.  \cite{movshovitz2017no} via the \texttt{Pytorch-Metric-Learning} library \cite{musgrave2020metric}, with no explicit triplet mining. Training was conducted on a single \texttt{NVIDIA A40} GPU. Due to the large number of object detections (approximately 4M cropped sub-images across 50 object categories), each epoch requires approximately 15 hours of computation, with GPU utilization analysis confirming that the available resources are fully saturated throughout training.
\begin{table}
\centering
\caption{Comparison across backbone architectures and the CLIP baseline (K=10).}
\label{tab:backbones}
\setlength{\tabcolsep}{4pt}
\renewcommand{\arraystretch}{1.08}
\begin{tabular}{r|cc|cc|cc}
\toprule
& \multicolumn{2}{c|}{\textbf{Label+Label}} & \multicolumn{2}{c|}{\textbf{Image+Label}} & \multicolumn{2}{c}{\textbf{Image+Image}} \\
\midrule

\textbf{Backbone} &
\makecell{Obj.\\P@$K$} &
\makecell{Date\\P@$K$} &
\makecell{Obj.\\P@$K$} &
\makecell{Date\\P@$K$} &
\makecell{Obj.\\P@$K$} &
\makecell{Date\\P@$K$} \\
\midrule

CLIP (arithmetic, see Appendix \ref{sec:clip_baseline_b}) & .152 & .115 & .387 & .076 & .389 & .102 \\
CLIP (prompting, see Appendix \ref{sec:clip_baseline_a}) & .415 & .105 & .372 & .079 & .520 & .081 \\
\midrule
VGG + TDIR & .397 & .712 & .388 & .579 & .669 & .234 \\
ResNet + TDIR & \textbf{.489} & .803 & \textbf{.444} & .665 & .721 & .306 \\
ConvNeXt + TDIR & .443 & .824 & .416 & .664 & \textbf{.725} & .393\\
ViT + TDIR& .413 & \textbf{.863} & .355 & \textbf{.708} & .646 & \textbf{.429} \\

\bottomrule
\end{tabular}
\end{table}

\subsection{Evaluation Protocol}
\label{sec:eval_protocol}
We evaluate each of the three inference modes (Section~\ref{sec:inference}) using a unified set of metrics
built around two axes: \emph{categorical fidelity} (does retrieval respect the object class?) and
\emph{temporal fidelity} (does retrieval respect the target date?).  As a standard practice in date estimation \cite{palermo2012dating,muller2017picture}, we consider a correct date classification whenever the retrieved date is within the same lustrum (5  years period).

For each inference mode, we retrieve the top-$K$ images
for every valid query and report two complementary precision metrics:
\begin{align}
    \textbf{Object Precision@}K &: \frac{\#\{x_i : c_i=c\}}{K}, \\
    \textbf{Date Precision@}K   &: \frac{\#\{x_i : y_i \approx y\}}{K},
\end{align}
where $y_i \approx y$ denotes temporal correctness within the same lustrum (5-year window).  Given that the DEW dataset contains images from 1930-1999, the random baseline is settled at $100\times\frac{1}{70 / 5} = 7.14\%$ chance of returning a correctly classified image in terms of the date estimation. In the label-based mode, $c$ and $y$ are taken directly from the query labels;
in the image-based modes, $c$ is determined by $c(x_c)$ and $y$ by the target year or
the transferred residual $r_y$ (see Section \ref{sec:inference}).

\section{Results}
In this section, we present an exhaustive evaluation of our framework in a highly applied setup. First, in Table \ref{tab:backbones}, we compare several widely used Computer Vision backbones against two CLIP baselines for reference (see Appendix \ref{sec:clip_baseline} for the baselines implementation). We observe that our proposed approach follows the general trends observed in computer vision, with ConvNeXt and ViT emerging as the top-performing and overall comparable architectures. The CLIP baselines prove to be competitive on retrieving the correct category, but are not sufficiently sensitive to temporal information despite its demonstrated sensitivity to image dates in \cite{sonthalia2026rankability} and \cite{barancova2023blind}.
\begin{table}[b]
\centering

\setlength{\tabcolsep}{7pt}
\renewcommand{\arraystretch}{1.08}
\begin{tabular}{ccc|cc|cc|cc}
\toprule
& & & \multicolumn{2}{c|}{\textbf{Label+Label}} & \multicolumn{2}{c|}{\textbf{Image+Label}} & \multicolumn{2}{c}{\textbf{Image+Image}} \\
\midrule
\textbf{Swap} &
\textbf{Aux} &
\textbf{Norm} &
\makecell{Obj.\\P@$K$} &
\makecell{Date\\P@$K$} &
\makecell{Obj.\\P@$K$} &
\makecell{Date\\P@$K$} &
\makecell{Obj.\\P@$K$} &
\makecell{Date\\P@$K$} \\
\midrule
\xmark & \cmark & \cmark & .362 & .712 & .342 & .568 & .683 & .358 \\
\cmark & \xmark & \cmark & .355 & .822 & .343 & \textbf{.676} & .691 &.354 \\
\cmark & \cmark & \xmark & .358 & .816 & .349 & .675 & .708 & .362 \\
\cmark & \cmark & \cmark &\textbf{ .443} & \textbf{.824} & \textbf{.416} & .664 & \textbf{.725} & \textbf{.393} \\
\bottomrule
\end{tabular}
\caption{Ablation study (K=10, \texttt{ConvNext}) for our method including the usage of the swapping trick, auxiliary proxies and normalization of the temporal component.}
\label{tab:ablation}
\end{table}

\begin{figure}
    \centering
    \includegraphics[width=0.99\linewidth]{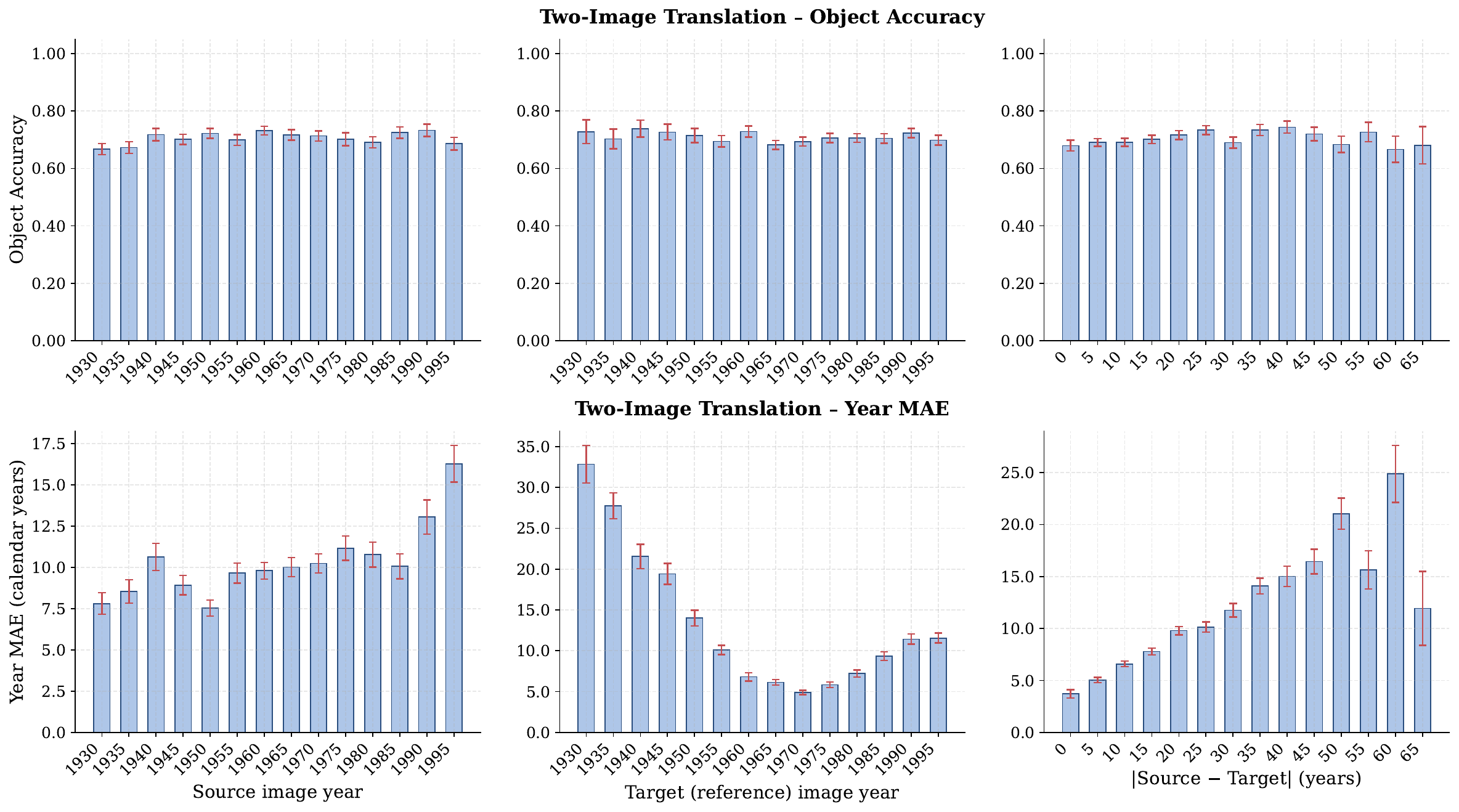}
    \caption{Calibration error bars when using query-by-examples year injection.}
    \label{fig:calibration_trans}
\end{figure}

Table \ref{tab:ablation} presents an ablation study. A key observation from the experiment is that the proposed swapping trick not only improves transitivity, as expected, but also leads to better date and object embedding subspaces. This is reflected in the improved performance observed not only for image-based queries but also for purely label-based inference, which does not rely on any transitive property induced by injecting a date into a temporal embedding. This observation partially supports the claim in Corollary \ref{cor:swapping_orthogonality} presented in Appendix \ref{sec:theorem_emergent}, heavily relies on the inclusion of a swapping term and states that this combination of losses imposes an implicit orthogonality regularization at the optimum.
\begin{table}[t]
\centering
\resizebox{.9\textwidth}{!}{
\small
\setlength{\tabcolsep}{4pt}
\renewcommand{\arraystretch}{1.2}
\begin{tabular}{>{\centering\arraybackslash}m{1cm} >{\centering\arraybackslash}m{1cm} | >{\centering\arraybackslash}m{1cm} >{\centering\arraybackslash}m{1cm} >{\centering\arraybackslash}m{1cm} >{\centering\arraybackslash}m{1cm} >{\centering\arraybackslash}m{1cm} >{\centering\arraybackslash}m{1cm}}
    \toprule
    \textbf{Category} & \textbf{Date} & \textbf{Top-1} & \textbf{Top-2} & \textbf{Top-3} & \textbf{Top-4} & \textbf{Top-5} & \textbf{Top-6} \\
    \midrule

    Poster &
    1975 &
    {\includegraphics[width=1cm]{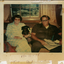}} &
    \greenbordered{\includegraphics[width=1cm]{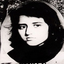}} &
    \greenbordered{\includegraphics[width=1cm]{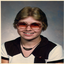}} &
    {\includegraphics[width=1cm]{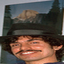}} &
    {\includegraphics[width=1cm]{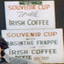}} &
    \redbordered{\includegraphics[width=1cm]{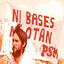}} \\[4pt]

    \includegraphics[width=1cm]{images/inference/qualitative_eval/poster_1980_6.png} &
    1980 &
    {\includegraphics[width=1cm]{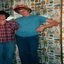}} &
    {\includegraphics[width=1cm]{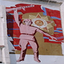}} &
    {\includegraphics[width=1cm]{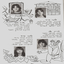}} &
    {\includegraphics[width=1cm]{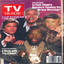}} &
    {\includegraphics[width=1cm]{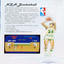}} &
    {\includegraphics[width=1cm]{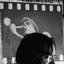}} \\[4pt]

    \includegraphics[width=1cm]{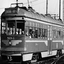} &
    \includegraphics[width=1cm]{images/inference/qualitative_eval/poster_1980_4.png} &
    {\includegraphics[width=1cm]{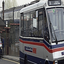}} &
    \greenbordered{\includegraphics[width=1cm]{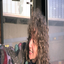}} &
    {\includegraphics[width=1cm]{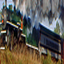}} &
    {\includegraphics[width=1cm]{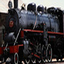}} &
    {\includegraphics[width=1cm]{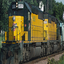}} &
    {\includegraphics[width=1cm]{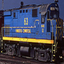}} \\[4pt]

    \includegraphics[width=1cm]{images/inference/qualitative_eval/poster_1980_4.png} &
    \includegraphics[width=1cm]{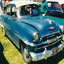} &
    \greenbordered{\includegraphics[width=1cm]{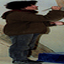}} &
    \greenbordered{\includegraphics[width=1cm]{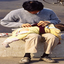}} &
    \greenbordered{\includegraphics[width=1cm]{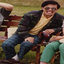}} &
    \greenbordered{\includegraphics[width=1cm]{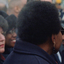}} &
    \greenbordered{\includegraphics[width=1cm]{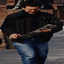}} &
    \greenbordered{\includegraphics[width=1cm]{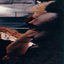}} \\[4pt]

    \includegraphics[width=1cm]{images/inference/qualitative_eval/car_1990.png} &
    \includegraphics[width=1cm]{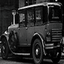} &
    \redbordered{\includegraphics[width=1cm]{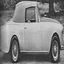}} &
    \redbordered{\includegraphics[width=1cm]{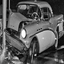}} &
    \redbordered{\includegraphics[width=1cm]{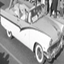}} &
    \redbordered{\includegraphics[width=1cm]{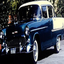}} &
    \redbordered{\includegraphics[width=1cm]{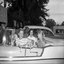}} &
    \redbordered{\includegraphics[width=1cm]{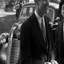}} \\[4pt]

    \includegraphics[width=1cm]{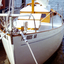} &
    \includegraphics[width=1cm]{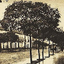} &
    {\includegraphics[width=1cm]{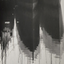}} &
    {\includegraphics[width=1cm]{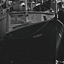}} &
    {\includegraphics[width=1cm]{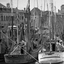}} &
    {\includegraphics[width=1cm]{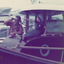}} &
    \redbordered{\includegraphics[width=1cm]{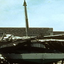}} &
    \redbordered{\includegraphics[width=1cm]{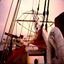}}\\
    \bottomrule
\end{tabular}}
\vspace{.25pt}

    \caption{Qualitative results with \categoryerror{ category error} and \temporalerror{temporal error}.}
    \label{tab:quaalitative_results}
\end{table}
The inclusion of an auxiliary term ($k_*$), instead of directly using the temporal proxy itself ($\mu^y_*$) in the swapping trick step, plays an important role in Image+Image inference, yielding significant gains in properly structuring the embedding space. However, when performing Image+Label inference, only a marginal performance gap is observed, likely due to the train–test mismatch at inference time caused by relying exclusively on the learned centroids ($\mu^y_*$). Lastly, the normalization of the temporal component appears to be the least impactful, although it provides a slight performance boost in Image+Image date estimation. Curiously, performance in the Label+Label object retrieval setup is maximized if and only if all three components are present in the training regime. Figure~\ref{fig:calibration_trans} reports object and date precision as a function of  source year, target year, and their absolute difference. Object precision remains stable  across all three axes, confirming that temporal transplantation does not corrupt the categorical content of the query: the category subspace is effectively insulated from the temporal injection. Date precision, by contrast, reveals a clear dependency on the target year and on the source-to-target gap, but notably \emph{not} on the source year itself. This asymmetry rules out a straightforward directional explanation: if the error were caused by the displacement vector $\mu^y$ acting unidirectionally, pushing representations forward or backward in time indiscriminately, one would expect a monotonic pattern with respect to direction. Instead, the error follows a U-shaped curve over the target year axis, with decades at the extremes of the temporal range being harder to reach regardless of where the source lies. Furthermore, precision degrades consistently as the temporal gap widens, suggesting that the optimisation does not perfectly  generalise large temporal displacements. Taken together, these findings indicate that the current additive translation mechanism, while effective at short and moderate temporal distances, is an approximation whose fidelity diminishes with displacement magnitude. A natural direction for future work is to replace the additive operator with a more expressive, potentially non-linear mechanism capable of modelling large temporal jumps more faithfully.

Although the numerical results may appear modest, a qualitative inspection in Table \ref{tab:quaalitative_results} reveals that the method generally behaves as expected. While errors in the predicted years are present, they mostly correspond to inaccuracies within the same lustrum rather than severe temporal artifacts (see Appendix \ref{sec:objectAnalysis} for a detailed analysis). This behavior may stem from two limitations. First, the temporal subspace is constrained to be object-agnostic, thus enforcing the same temporal displacement across all object categories. As introduced in Section \ref{sec:intro}, not every object evolves at the same pace, which may inherently limit the achievable performance under this design choice. In many cases, such as the fourth row, the temporal transplant is qualitatively correct, but the retrieved samples exhibit category ambiguity due to images lying near the boundary between two classes. In that example, the target category is ``poster'', yet images belonging to the ``men'' category are retrieved while still exhibiting the correct temporal attribution. As discussed, the error tends to increase with temporal distance, which can be observed in the fifth row, where the direction of the temporal shift is correct but the magnitude of the displacement is insufficient. Lastly, in some cases, such as the final row, the target category may simply contain too few old samples in the database because the object itself is inherently modern. As a result, the retrieved samples belong to the correct category but remain more recent than desired.

\section{Conclusions}
In this paper, we have presented TDIR, a representation learning framework that formalises the decomposition of historical image embeddings into orthogonal temporal and categorical subspaces. We have established the theoretical conditions under which such a decomposition holds, proved that the required orthogonality emerges naturally from joint optimisation, and characterised the error incurred when those conditions are only partially met. To our knowledge, this constitutes the first formal treatment of temporal decomposability in visual embeddings, addressing a limitation of prior work that either observes temporal structure post-hoc or exploits it implicitly without a principled formulation. Beyond the theoretical contributions, we have grounded the framework in the novel problem of Composed Historical Image Retrieval, where a query simultaneously specifies object content and a target time period. This setting reflects a natural and practically relevant archival task, yet no existing benchmark supported its evaluation. We address this by extending the DEW dataset with object-level detections across 50 categories, enabling the first evaluation of compositional retrieval over a diverse photographic archive covering seven decades. 

Limitations include the assumption of a shared temporal displacement across all object categories, which may not reflect the uneven pace at which different objects evolve historically, and the use of an additive transplantation operator whose fidelity diminishes at large temporal distances. Future work will explore non-linear temporal operators and category-specific temporal subspaces, as well as the extension of the benchmark to support zero-shot evaluation over unseen object categories.
\section*{Acknowledgments}
{\footnotesize
This work has been partially supported by the Spanish project PID2024-157778OB-I00, Ministerio de Ciencia e Innovación, the Departament de Cultura of the Generalitat de Catalunya, and the CERCA Program. Adrià Molina is funded with the PRE2022-101575 grant provided by MCIN / AEI / 10.13039 / 501100011033 and by ERDF/EU.}

\clearpage
{\footnotesize
\bibliography{egbib}
}
\clearpage
\appendix
\section*{Supplementary Material}

This document contains the supplementary material for the BMVC submission:

\begin{center}
    \textbf{Composed Historical Image Retrieval by Modeling Temporal Representations}
\end{center}

The supplementary material is organized as follows:

\vspace{0.5em}

\begin{table}[h]
\centering
\begin{tabular}{p{0.15\linewidth} p{0.75\linewidth}}
\hline
\textbf{Section} & \textbf{Content} \\
\hline
A & Derivations and Demonstrations \\
B & Photographic Material Calculation \\
C & Object-Centric Analysis \\
D & CLIP Baselines \\
E & Frequently Asked Questions (FAQs) \\

\hline
\end{tabular}
\end{table}

\vspace{1em}

The following sections provide theoretical derivations, empirical analyses,
and additional experimental details supporting the main manuscript.
\clearpage
\section{Derivations for Decomposable Representations}

\subsection{Chain rule of mutual information}
\label{sec:chain_rule}
Because the following derivations will heavily rely on rearranging terms in equalities using the chain rule, let us consider the definition on a general setting, using three random variables: $A, B, \text{and } C$.

The chain rule of mutual information states that:
\begin{equation}
   I(A,C;B) = I(A;B) + I(C;B \mid A) 
\end{equation}

Intuitively, this means that the information that the joint representation \((A,C)\) carries about \(B\) can be decomposed into the information provided by \(A\), plus the additional information provided by \(C\), after conditioning on \(A\).

\subsection{Temporally Decomposable Image Representations (TDIR)}
\label{sec:tfir_demonstraction}

As expressed in the main corpus of the manuscript, we note that category and year classes are decomposable as a design choice, which means that the expression in Definition~\ref{def:definition_classes_assumption}

\begin{equation}
p(\mu^c, c, \mu^y, y)=p(\mu^c, c)\, p(\mu^y, y)
\label{eq:independence_assumption}
\end{equation}

is satisfied. In short, the class centroids are distributed with no influence from the date information and vice-versa.

\begin{proposition}[Temporally Decomposable Image Representation]
Under the model assumptions of Definition~\ref{def:definition_classes_assumption}, we note that the following equality for the image representation holds:

\begin{equation}
    I(\mu^c, \mu^y;\, c,y) = I(\mu^c;\, c) + I(\mu^y;\, y)
\end{equation}

\end{proposition}

\begin{proof}
We apply the chain rule for mutual information to decompose the left-hand side:
\begin{equation}
I(\mu^c, \mu^y;\, c,y)
=
\underbrace{I(\mu^c;\, c,y)}_{\text{first term}}
+
\underbrace{I(\mu^y;\, c,y \mid \mu^c)}_{\text{second term}}
\end{equation}

We apply again the chain rule to $I(\mu^c;\, c,y)$:
\begin{equation}
    I(\mu^c;\, c,y) = I(\mu^c;\, c) + I(\mu^c;\, y \mid c)
\end{equation}
From Definition~\ref{def:definition_classes_assumption}, the category centroid does not provide information about the date of the image, therefore it does not provide information when conditioned on the category itself:
\begin{equation}
\label{eq:simplification_first_term}
I(\mu^c;\, c,y)=I(\mu^c;\, c)+\cancel{I(\mu^c;\, y \mid c)}= I(\mu^c;\, c)
\end{equation}

Hence the first term simplifies to $I(\mu^c;\, c)$ if and only if the model assumption ($\mu^c \perp y \mid c$) holds. Similarly, the term $I(\mu^y;\, c,y \mid \mu^c)$ is expanded again by applying the chain rule of conditional mutual information:
\begin{equation}
\label{eq:simplification_second_term}
I(\mu^y;\, c,y \mid \mu^c)=I(\mu^y;\, c \mid \mu^c)+I(\mu^y;\, y \mid \mu^c,c)
\end{equation}

From the Markov structure that emerges from Eq. ~\eqref{eq:independence_assumption},
\begin{equation}
\mu^c \rightarrow c
\qquad\text{and}\qquad
\mu^y \rightarrow y,
\end{equation}

together with the independence assumption $\mu^y \perp (\mu^c,c)$, it follows that
\begin{equation}
I(\mu^y;\, c \mid \mu^c)=0,
\end{equation}
and conditioning on $(\mu^c,c)$ does not affect the dependence between $\mu^y$ and $y$, yielding
\begin{equation}
I(\mu^y;\, y \mid \mu^c,c)
=
I(\mu^y;\, y).
\end{equation}
Therefore,
\begin{equation}
I(\mu^y;\, c,y \mid \mu^c)
=
I(\mu^y;\, y).
\end{equation}

Combining the results in Equations \eqref{eq:simplification_first_term} and \eqref{eq:simplification_second_term}:
\begin{equation}
I(\mu^c, \mu^y;\, c,y)
=
\underbrace{I(\mu^c;\, c,y)}_{I(\mu^c;\, c) + 0}
+
\underbrace{I(\mu^y;\, c,y \mid \mu^c)}_{I(\mu^y;\, y)} = I(\mu^c;\, c) + I(\mu^y;\, y)
\end{equation}
which holds if and only if the initial assumption in Eq.~\eqref{eq:independence_assumption} is satisfied. 

\end{proof}

\begin{lemma}[Orthogonality of Centroid Subspaces]
\label{lemma:orthogonality}
Under Definition~\ref{def:definition_classes_assumption},
the category and temporal centroids are orthogonal:
\begin{equation}
    \operatorname{Cov}(\mu^c, \mu^y) = \mathbf{0},
\end{equation}
and the total representational variance decomposes as
\begin{equation}
    \operatorname{Var}(\mu^c + \mu^y)
    = \operatorname{Var}(\mu^c) + \operatorname{Var}(\mu^y).
\end{equation}
\end{lemma}

\begin{proof}
By Definition~\ref{def:definition_classes_assumption}, marginalising over $(c, y)$ and applying Fubini's theorem:
\begin{equation}
p(\mu^c, \mu^y)
= \int\!\!\int p(\mu^c, c)\cdot p(\mu^y, y)\,dc\,dy
= \underbrace{\int p(\mu^c, c)\,dc}_{p(\mu^c)}
  \cdot
  \underbrace{\int p(\mu^y, y)\,dy}_{p(\mu^y)}
= p(\mu^c)\cdot p(\mu^y),
\end{equation}
so $\mu^c \perp \mu^y$. Independence implies
$\operatorname{Cov}(\mu^c, \mu^y) = \mathbf{0}$,
and the variance decomposition follows immediately from the bilinearity of covariance.
\end{proof}

\subsection{Estimating the error in real scenarios}
\label{sec:tfir_demonstraction_error}
Because the total independence $\mu^c \perp \mu^y$ (Lemma~\ref{lemma:orthogonality}) cannot be guaranteed in practical terms, where data might contain severe spurious correlations, one must account for the incorporation of a noise vector $\varepsilon \in\mathbb{R}^d$ which entangles both year and category information:
\begin{equation}
v = f(x) = \mu^c + \mu^y+\epsilon \quad \text{and} \quad I(\epsilon;\, c,y) \neq 0
\end{equation}
Our model assumption is that epsilon can emerge due to correlation on the training data (hence, the labels $y$ and $c$) but that an \textbf{independent proxy vector can be learnt} despite this entanglement\footnote{With this, we do not negate the presence of noise and spurious correlations on the training; $\epsilon$ can heavily depend on the target labels, but one can impose a model restriction where $\mu^c$ and $\mu^y$ are learned to be orthogonal vectors.}:
\begin{equation}
\epsilon \perp (\mu^c, \mu^y)
\end{equation}

\begin{proposition}[Practical Learnability of TDIRs]
Under the TDIR framework, the discrepancy term $\delta \triangleq I(\varepsilon;\, c, y)$ 
depends only on the statistical structure of the training data and not on the 
choice of proxy vectors $\mu^c$ or $\mu^y$. 
\begin{equation}
I(v;\, c,y)
\leq
I(\mu^c;\, c)
+
I(\mu^y;\, y)
+
\delta \qquad \text{where} \qquad \delta
\triangleq
I(\varepsilon;\, c,y)
\end{equation}
Consequently, $\delta$ constitutes a 
data-dependent constant with respect to the optimization, and the category and year 
centroids can be learned freely to satisfy the orthogonality constraint 
$\langle \mu^c,\, \mu^y \rangle = 0$ without affecting this error bound.
\end{proposition}

\begin{proof}
The real mutual information expression entangles $\epsilon$ with the label variables:
\begin{equation}
I(\mu^c, \mu^y, \epsilon;\, c,y) = I(\mu^c, \mu^y;\, c,y) + I(\epsilon;\, \mu^c, \mu^y\mid c,y)
\end{equation}
Because the model is assumed to be trained under an independence condition $\epsilon \perp (\mu^c, \mu^y)$,
\begin{equation}
I(\epsilon;\, \mu^c, \mu^y\mid c,y) = I(\epsilon;\, \mu^c, \mu^y)
\end{equation}
Therefore,
\begin{equation}
I(\mu^c, \mu^y, \epsilon;\, c,y) = \underbrace{I(\mu^c, \mu^y;\, c,y)}_{I(\mu^c;\, c) + I(\mu^y;\, y)} + I(\epsilon;\, \mu^c, \mu^y)
\end{equation}
To obtain the total expression for $I(v;\, c,y)$, we apply the \textbf{chain rule of mutual information}:
\begin{equation}
I(\underbrace{v}_{A},\, \underbrace{\mu^c, \mu^y, \epsilon}_{C};\, \underbrace{c,y}_{B})
= I(v;\, c,y) + I(\mu^c, \mu^y, \epsilon;\, c,y \mid v)
\end{equation}
Since $v = \mu^c + \mu^y + \epsilon$ is a function of $(\mu^c, \mu^y, \epsilon)$,
the joint tuple $(v, \mu^c, \mu^y, \epsilon)$ carries no more information about $(c,y)$
than $(\mu^c, \mu^y, \epsilon)$ alone:
\begin{equation}
I(v, \mu^c, \mu^y, \epsilon;\, c,y) = I(\mu^c, \mu^y, \epsilon;\, c,y)
\end{equation}
Isolating $I(v;\, c,y)$:
\begin{equation}
\label{eq:issolated_mutual_with_epsilon}
I(v;\, c,y) = I(\mu^c, \mu^y, \epsilon;\, c,y) - I(\mu^c, \mu^y, \epsilon;\, c,y \mid v)
\end{equation}

Substituting the expression in Eq. \eqref{eq:issolated_mutual_with_epsilon}:
\begin{equation}
\boxed{I(v;\, c,y) = I(\mu^c;\, c) + I(\mu^y;\, y) + I(\epsilon;\, c,y) - I(\mu^c, \mu^y, \epsilon;\, c,y \mid v)}
\end{equation}
where the last term $I(\mu^c, \mu^y, \epsilon;\, c,y \mid v)$ represents the information
lost by collapsing the three components into a single vector $v$. Identifying $\delta$:
\begin{equation}
\delta = I(\epsilon;\, c,y) - \underbrace{I(\mu^c, \mu^y, \epsilon;\, c,y \mid v)}_{\geq\, 0}
\end{equation}
we can rewrite the exact relation compactly as:
\begin{equation}
I(v;\, c,y) = I(\mu^c;\, c) + I(\mu^y;\, y) + \underbrace{\delta}_{\leq\, I(\epsilon;\,c,y)}
\end{equation}
Since \textbf{mutual information is always non-negative}, $\delta \leq I(\varepsilon;\, c,y)$
provides an upper bound. To obtain a more precise approximation of $\delta$,
we expand $I(\varepsilon;\, c,y)$ via the chain rule:
\begin{equation}
I(\varepsilon;\, c,y) = I(\varepsilon;\, c) + I(\varepsilon;\, y \mid c)
\end{equation}
Under the assumption that category $c$ and year $y$ are independent,
$I(\varepsilon;\, y \mid c) \approx I(\varepsilon;\, y)$, giving:
\begin{equation}
\delta \leq I(\varepsilon;\, c) + I(\varepsilon;\, y)
\end{equation}
\end{proof}

This shows that $\delta$ is controlled by how much the noise $\varepsilon$
leaks information about each label separately. If the model successfully
disentangles $\mu^c$ and $\mu^y$ from $\varepsilon$, both terms vanish
and $\delta \to 0$, recovering the ideal case:
\begin{equation}
I(v;\, c,y) \approx I(\mu^c;\, c) + I(\mu^y;\, y)
\end{equation}
Thus, under the assumption that such orthogonal centroids can be learned (see Figure~\ref{fig:graphical_model}), the ideal mutual information equality in the practical case is consistent with the TDIR proposition up to an upper bound $\delta$ characterized by the entanglement of training data.
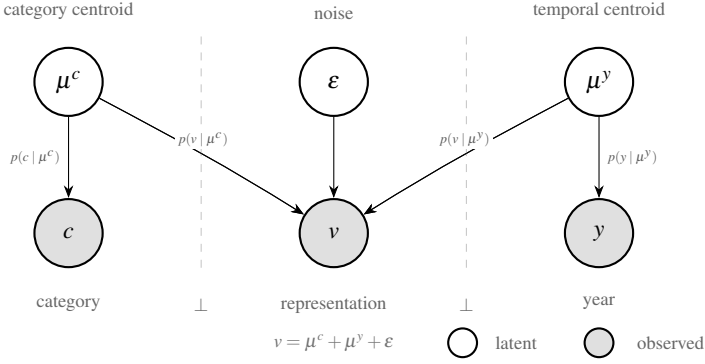
\begin{figure}
    \centering
    \begin{tikzpicture}[
  latent/.style={
    circle, draw, thick,
    minimum size=0.9cm, inner sep=0pt, font=\small
  },
  obs/.style={
    circle, draw, thick, fill=black!12,
    minimum size=0.9cm, inner sep=0pt, font=\small
  },
  lbl/.style={font=\scriptsize, text=black!60},
  elbl/.style={font=\tiny, text=black!65, inner sep=1pt, fill=white},
  edge/.style={->, >=Stealth, thin},
  sep/.style={dashed, black!22, thin},
]
 
\node[latent] (muc) at (0.0,  0)    {$\mu^c$};
\node[latent] (eps) at (3.5,  0)    {$\varepsilon$};
\node[latent] (muy) at (7.0,  0)    {$\mu^y$};
 
\node[obs]    (c)   at (0.0, -2.0)  {$c$};
\node[obs]    (v)   at (3.5, -2.0)  {$v$};
\node[obs]    (y)   at (7.0, -2.0)  {$y$};
 
\draw[edge] (muc) --
  node[elbl, left=2pt] {$p(c\mid\mu^c)$} (c);
\draw[edge] (muc) to[out=-28, in=148] (v);
\draw[edge] (eps) -- (v);
\draw[edge] (muy) to[out=-152, in=32]  (v);
\draw[edge] (muy) --
  node[elbl, right=2pt] {$p(y\mid\mu^y)$} (y);
 
\draw[sep] (1.75,  0.60) -- (1.75, -2.55);
\draw[sep] (5.25,  0.60) -- (5.25, -2.55);
 
\node[lbl, above=7pt of muc] {category centroid};
\node[lbl, above=7pt of eps] {noise};
\node[lbl, above=7pt of muy] {temporal centroid};
 
\node[lbl, below=7pt of c]   {category};
\node[lbl, below=7pt of v]   {representation};
\node[lbl, below=7pt of y]   {year};
 
\node[lbl] at (1.75, -2.95) {$\perp$};
\node[lbl] at (5.25, -2.95) {$\perp$};
 
\node[font=\scriptsize, text=black!55] at (3.5, -3.45)
  {$v = \mu^c + \mu^y + \varepsilon$};
 
\node[latent, minimum size=0.38cm]        (ll) at (5.2, -3.45) {};
\node[lbl, right=3pt of ll]                    {latent};
\node[obs,    minimum size=0.38cm]        (lo) at (7, -3.45) {};
\node[lbl, right=5pt of lo]                    {observed};
 \draw[edge] (muc) to[out=-28, in=148]
  node[elbl, above=2pt] {$p(v\mid\mu^c)$} (v);
\draw[edge] (muy) to[out=-152, in=32]
  node[elbl, above=2pt] {$p(v\mid\mu^y)$} (v);
\end{tikzpicture}
    \caption{Resulting graphical model from Proposition~\ref{prop:learnability}}
    \label{fig:graphical_model}
\end{figure}

The joint optimization of $\mathcal{L}^c$ and $\mathcal{L}^y$ implicitly enforces 
the orthogonality condition $\langle \mu^c,\, \mu^y \rangle = 0$ required by the 
TDIR definition, without imposing it as an explicit constraint. The proof relies 
on a preliminary result about the structure of the year proxies, which we establish first.
\clearpage
\subsection{TDIR as an Emergent Property of the Joint Loss}
\label{sec:theorem_emergent}
With all the properties and propositions derived above, we can then state that the following theorem holds:
\begin{theorem}[Emergent TDIR]
\label{thm:theorem_emergent}
Let $\mathcal{L}^c$ and $\mathcal{L}^y$ be the proxy loss functions defined over an 
embedding space $\mathbb{R}^d$. Then:

\textbf{Part I (exact case).} If $I(\epsilon\,;\,c,y) = 0$, joint minimization of 
$\mathcal{L} = \mathcal{L}^c + \mathcal{L}^y$ implies TDIR:
\begin{equation}
    \min_\theta\,\mathcal{L}^c + \mathcal{L}^y \implies 
    p(\mu^c, \mu^y, c, y) = p(\mu^c, c)\cdot p(\mu^y, y)
\end{equation}

\textbf{Part II (asymptotic case).} If $I(\epsilon\,;\,c,y) \neq 0$, joint 
minimization implies asymptotic TDIR:
\begin{equation}
    \min_\theta\,\mathcal{L}^c + \mathcal{L}^y \implies 
    \langle\mu^c,\,\mu^y\rangle \to O\!\left(\frac{I(\epsilon\,;\,c,y)}{\|\mu^y\|^2}
    \right) \quad \forall\,c,\,y
\end{equation}
\end{theorem}

\begin{proof}

\textbf{Part I: $I(\epsilon\,;\,c,y) = 0 \implies$ TDIR.}

The category proxy loss is minimized when the score $v^\top\mu^c$ is maximal and 
uniform across all years. Since $v_A^c = \mu^c + \mu^{y_A}$ and 
$v_B^c = \mu^c + \mu^{y_B}$ share the same category but differ in year, the 
minimum requires:
\begin{equation*}
    v_A^\top\mu^c = v_B^\top\mu^c \quad \forall\, y_A, y_B 
    \implies 
    \langle\mu^{y_A},\mu^c\rangle = \langle\mu^{y_B},\mu^c\rangle \quad \forall\, y_A,y_B
\end{equation*}
hence $I(\mu^c\,;\,y) = 0$. For the temporal loss, the shared displacement 
vectors $k^y$ impose, for any two categories $c \neq c'$ sharing the same years:
\begin{equation*}
    r_A^c + k^y_B \approx \mu^{y_B}, \qquad r_A^{c'} + k^y_B \approx \mu^{y_B}
\end{equation*}
Subtracting yields $r_A^c \approx r_A^{c'}$, so the score 
$r_A^c{}^\top\mu^y$ is independent of the category, hence $I(\mu^y\,;\,c) = 0$.
Applying the chain rule of mutual information:
\begin{equation*}
    I(\mu^c,\mu^y\,;\,c,y) = I(\mu^c\,;\,c) + 
    \underbrace{I(\mu^y\,;\,c\mid\mu^c)}_{=\,0} + 
    \underbrace{I(\mu^c\,;\,y\mid\mu^y,c)}_{=\,0} + 
    I(\mu^y\,;\,y) = I(\mu^c\,;\,c) + I(\mu^y\,;\,y)
\end{equation*}
This equality is equivalent, via the KL divergence and Fubini's theorem, to the 
factorization $p(\mu^c,\mu^y,c,y) = p(\mu^c,c)\cdot p(\mu^y,y)$.

\textbf{Part II: $I(\epsilon\,;\,c,y) \neq 0 \implies$ asymptotic TDIR.}

\medskip

The proxy loss $\mathcal{L}(u, \mu^*; \mathcal{P}) = -\log\,\sigma(u^\top\mu^*)$
is a softmax cross-entropy, which is smooth and strictly convex in a neighborhood
of its minimum. A second-order Taylor expansion around the minimizer yields a
local quadratic approximation of the form $\|u - \mu^*\|^2 + O(\|u - \mu^*\|^3)$.
We therefore adopt the squared-error surrogates as a standard local approximation:

\begin{align}
\mathcal{L}^c(T) &\;\propto\; \|v_A - \mu^c\|^2 + \|v_B - \mu^c\|^2,
\label{eq:euc_loss_c}\\
\mathcal{L}^y(T) &\;\propto\; \|\hat{r}_A - \mu^{y_A}\|^2 
+ \|\hat{r}_B - \mu^{y_B}\|^2 
+ \|\hat{r}_{A\to B} - \mu^{y_B}\|^2 
+ \|\hat{r}_{B\to A} - \mu^{y_A}\|^2.
\label{eq:euc_loss_y}
\end{align}

\noindent
This approximation is tight near convergence, where embeddings cluster around
their respective proxies, and the conclusions below hold to the same asymptotic
order as the Taylor remainder.

\medskip
Decompose $\mu^c$ over the temporal subspace as 
$\mu^c = \tilde{\mu}^c + \alpha\,\mu^{y_A} + \beta\,\mu^{y_B}$ with 
$\tilde{\mu}^c \perp \mu^y$. Under $v = \mu^c + \mu^y + \epsilon$, the year 
residuals are:
\begin{equation*}
    r_A^c = (1-\alpha)\mu^{y_A} - \beta\mu^{y_B} + \epsilon_A^c, \qquad
    r_B^c = (1-\beta)\mu^{y_B} - \alpha\mu^{y_A} + \epsilon_B^c
\end{equation*}
Substituting into the loss surrogates to $\mathcal{L}^y$ and expanding:
\begin{equation*}
\begin{aligned}
\mathcal{L}^y \propto\;&+ \|\epsilon_A^c\|^2 + \|\epsilon_B^c\|^2
- 2\alpha\langle\mu^{y_A},\epsilon_A^c+\epsilon_B^c\rangle - 2\beta\langle\mu^{y_B},\epsilon_A^c+\epsilon_B^c\rangle+
\\&+ \underbrace{
2\alpha^2\|\mu^{y_A}\|^2
+ 2\beta^2\|\mu^{y_B}\|^2
+ 4\alpha\beta\langle\mu^{y_A},\mu^{y_B}\rangle
}_{Q(\alpha,\beta)\geq0}
\\
\end{aligned}
\end{equation*}
The quadratic form $Q(\alpha,\beta)$ is positive semidefinite by 
Cauchy-Schwarz. The cross terms displace the minimum from $(0,0)$ to:
\begin{equation*}
    \alpha^* = \frac{\langle\mu^{y_A},\,\epsilon_A^c+\epsilon_B^c\rangle}
    {2\|\mu^{y_A}\|^2} + O(\epsilon^2), \qquad 
    \beta^* = \frac{\langle\mu^{y_B},\,\epsilon_A^c+\epsilon_B^c\rangle}
    {2\|\mu^{y_B}\|^2} + O(\epsilon^2)
\end{equation*}
For small correlations, the covariance between $\epsilon$ and $\mu^y$ satisfies 
$\mathrm{Cov}(\epsilon,\mu^y) \approx \|\epsilon\|\cdot\|\mu^y\|\cdot
\sqrt{2\,I(\epsilon\,;\,c,y)}$, which gives:
\begin{equation*}
    \langle\mu^c,\,\mu^y\rangle = \alpha^*\|\mu^y\|^2 + O(\epsilon^2) 
    = O\!\left(\frac{I(\epsilon\,;\,c,y)}{\|\mu^y\|^2}\right)
\end{equation*}
Global consistency follows from the shared $k^y$: a single $k^y_B$ cannot 
simultaneously compensate projections $\alpha \neq \alpha'$ across categories 
unless $(\alpha-\alpha')\mu^{y_A} = \epsilon_A^{c'} - \epsilon_A^c$, which 
cannot hold for all category pairs since $\epsilon$ is not a control variable 
of the optimizer. The unique globally consistent solution is therefore:
\begin{equation*}
    \langle\mu^c,\,\mu^y\rangle \to 
    O\!\left(\frac{I(\epsilon\,;\,c,y)}{\|\mu^y\|^2}\right) 
    \quad \forall\,c,y \qedhere
\end{equation*}

\end{proof}
\begin{corollary}[Swapping Trick Enforces Cross-Category Orthogonality]
\label{cor:swapping_orthogonality}

Under the conditions of Theorem~\ref{thm:theorem_emergent}, the shared 
displacement vectors $k^y$ enforce $\langle \mu^c, \mu^y \rangle \to 0$ 
uniformly across all categories.

\medskip

Since the same $k^y_B$ is used when translating residuals from any category, 
satisfying the swapping constraint simultaneously for two categories 
$c \neq c'$ with residual projections $\alpha \neq \alpha'$ onto $\mu^y$ 
would require:
\begin{equation}
    (\alpha - \alpha')\,\mu^{y_A} \;=\; \epsilon_A^{c'} - \epsilon_A^c
    \qquad \forall\; c,\, c'.
\end{equation}

\medskip

This cannot hold globally, since $\epsilon$ is not a control variable of 
the optimizer. The unique globally consistent solution is therefore 
$\alpha = \alpha' = 0$ for all $c,\, c'$, giving:
\begin{equation}
    \mathrm{Proj}_{\mu^y}(\mu^c) \;\to\; 0 \qquad \forall\; c,\, y.
\end{equation}

\end{corollary}

\section{Photographic material calculation}
\label{sec:how_many_photos}

To estimate the proportion of photographic material in large archival collections, we surveyed three major public archives.

\paragraph{Archives Nationales (France).} Using the publicly available digitisation corpus dataset~\cite{archivesnationales_dataset}, a small available sample drawn from the 2025 digitisation batch was analysed. Of the total documents, 123 were non-textual, representing 23\% of the sample ($123/0.23 \approx 535$ documents in total). Among those non-textual items, 68 were strictly photographic, accounting for \textbf{approximately 13\%} of the full sample. 

\paragraph{Library of Congress (USA).} The collection can be divided into digitised and non-digitised holdings~\cite{loc_search}. Among the \emph{non-digitised} material, the catalogue lists 16M books, 3M newspapers, 1M periodicals, 0.5M manuscripts, 1.3M photographic and pictorial items, and 450M audiovisual items. Among the \emph{digitised} material, it lists 3M newspapers, 0.7M books, 0.5M manuscripts, 0.5M periodicals, and 1.2M photographs. Focusing on the textual-adjacent holdings (i.e.\ excluding the bulk audiovisual collection, which is a category of its own), photographic material represents:
\[
    \frac{1.3 + 1.2}{16 + 3 + 1 + 0.5 + 1.3 + 3 + 0.7 + 0.5 + 0.5 + 1.2} \approx \frac{2.5}{27.7} \approx \mathbf{9\%}.
\]
\paragraph{Portal de Archivos Españoles (Spain).} The portal reports approximately 40 million digital objects in total~\cite{pares_estadisticas}. According to~\cite{anuario2025}, roughly 2 million of these are classified as photographic objects (excluding videos, posters, and other non-textual formats), yielding a photographic fraction of approximately $2/40 = \mathbf{5\%}$.

\section{Object-Centric Analysis}
\label{sec:objectAnalysis}

The per-object analysis reveals two complementary sources of error in the composed retrieval framework, illustrated in Figure~\ref{fig:error_per_object}. When considering the \textit{reference} image (the image from which temporal information is extracted), objects with low Year MAE such as tires, shorts, glasses, and human faces act as reliable temporal vessels: their embeddings carry a strong and unambiguous date signal, suggesting minimal entanglement between their categorical and temporal subspaces. Conversely, objects such as wheels, vehicle registration plates, dresses, buildings, and men exhibit higher reference error, indicating that their visual representations encode date information less cleanly, likely due to higher intra-class visual variance across decades. On the \textit{target} side (the object whose category should be preserved after temporal transplantation), objects like posters, girls, trains, houses, and dresses are retrieved with relatively low temporal error, while ladders, human faces, tires, wheels, and hats show the highest MAE. This asymmetry suggests that temporally ambiguous objects are doubly penalised: they are poor sources of date information and poor recipients of temporal injection.

\begin{figure}[t]
    \centering
    \includegraphics[width=0.95\linewidth]{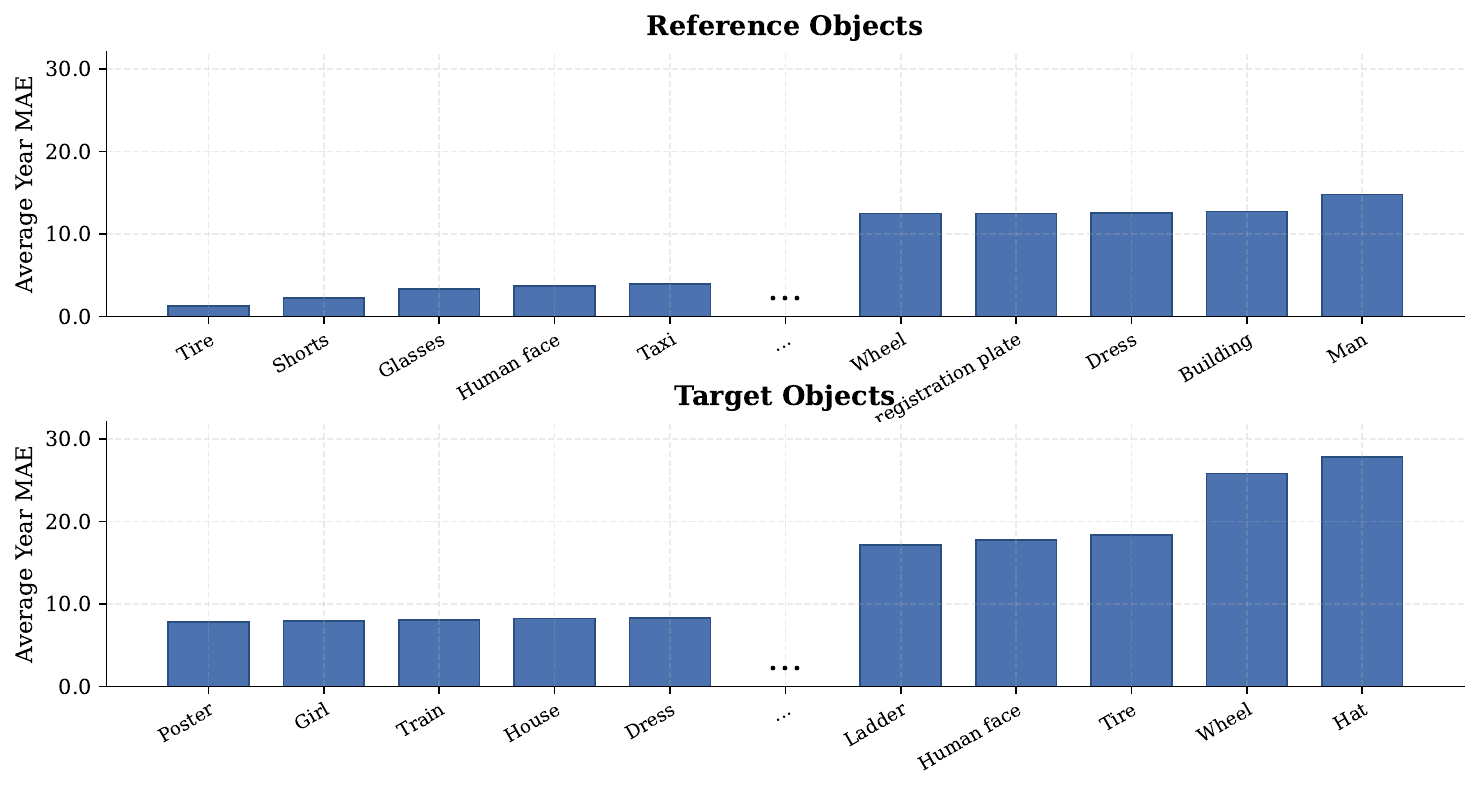}
    \caption{Average Year Error decomposed by reference object (top) and target object (bottom) in the Image+Image inference mode. Reference objects with low error are reliable vessels for date injection, indicating clean category representations with low temporal entanglement. Target objects with high error are temporally ambiguous, making them difficult to anchor to a specific period regardless of the source image used.}
    \label{fig:error_per_object}
\end{figure}

Figure~\ref{fig:radars_object_prec} reports per-object precision across all three inference modes using radar plots, for both object retrieval (top row) and date estimation (bottom row). In the object retrieval setting, the largest errors concentrate on categories that are visually sub-specific or easily confused with related classes: glasses, vehicle registration plates, and items of clothing such as socks and footwear tend to share visual features with broader categories, creating spurious correlations that degrade categorical fidelity. For date estimation (bottom row), the Label+Label mode achieves strong and consistent performance across virtually all object categories, confirming that the proxy-based subspace structure encodes temporal information reliably when labels are directly available. However, when queries are issued through images, performance degrades selectively for objects that lack distinctive temporal signatures, such as generic architectural elements or accessories, producing a marked and category-specific drop in precision.

\begin{figure}[t]
    \centering
    \includegraphics[width=0.32\linewidth]{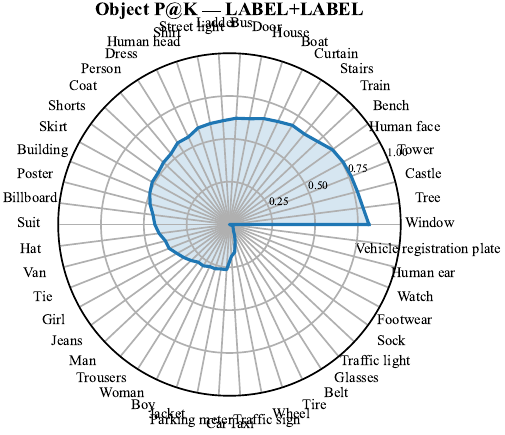}
    \includegraphics[width=0.32\linewidth]{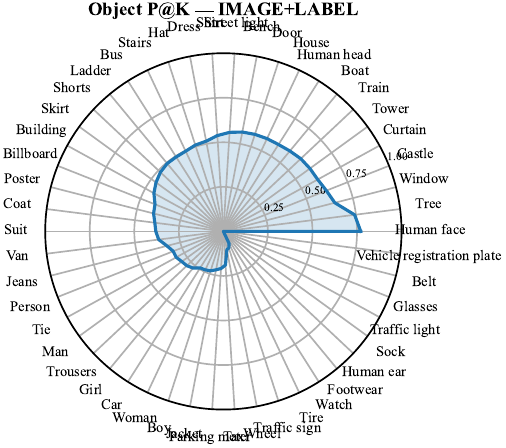}
    \includegraphics[width=0.32\linewidth]{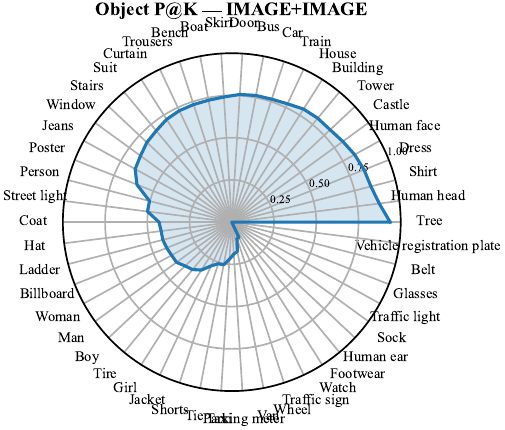} \\[4pt]
    \includegraphics[width=0.32\linewidth]{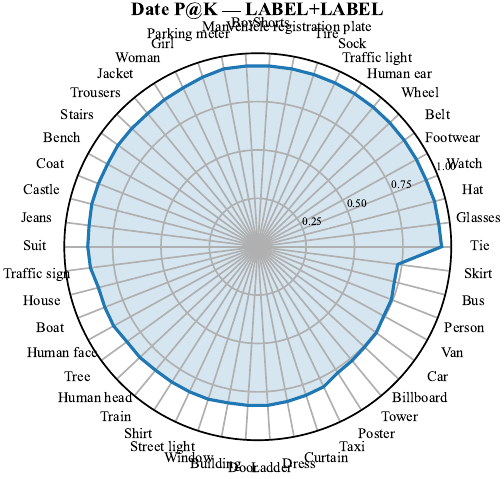}
    \includegraphics[width=0.32\linewidth]{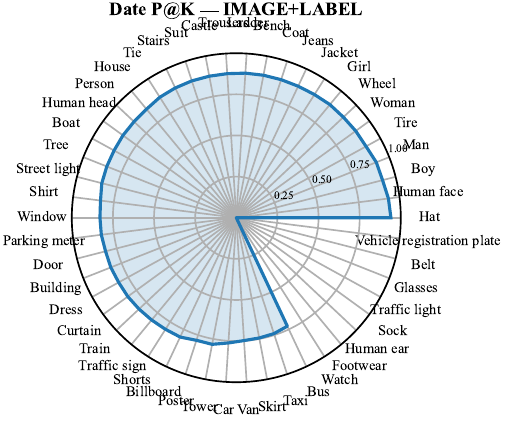}
    \includegraphics[width=0.32\linewidth]{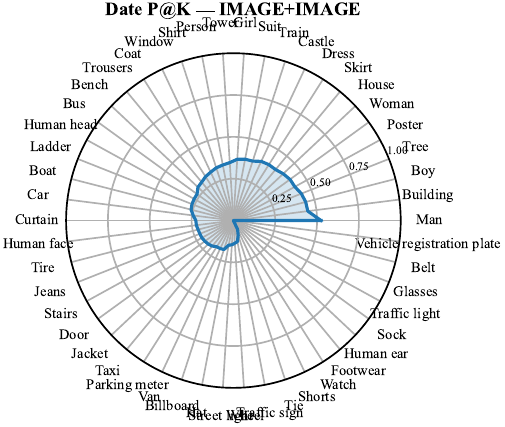}
    \caption{Per-object precision radar plots for object retrieval (top row) and date estimation (bottom row) across Label+Label, Image+Label, and Image+Image inference modes (left to right). Object retrieval errors concentrate on visually sub-specific categories such as glasses, vehicle registration plates, and footwear, which share features across class boundaries. For date estimation, Label+Label achieves consistently strong performance, while image-based modes reveal a pronounced and category-dependent performance gap for temporally ambiguous objects.}
    \label{fig:radars_object_prec}
\end{figure}

The temporal behaviour of each inference mode is further analysed in Figure~\ref{fig:error_per_year}. Date estimation precision degrades toward earlier decades for both Image+Label and Image+Image modes, consistent with the sparser representation of pre-1950 material in the DEW dataset. Notably, the Image+Image mode exhibits a reversal of the general trend observed in label-based settings: whereas date precision typically exceeds object precision under Label+Label and Image+Label queries, the Image+Image mode yields stronger object than date precision. This suggests that the categorical subspace is more robust to temporal transplantation than the temporal subspace is to cross-image injection, a finding aligned with the theoretical guarantee of Proposition~3, which ensures category preservation under orthogonal decomposition but imposes no analogous bound on the fidelity of the transplanted temporal signal.

\begin{figure}[t]
    \centering
    \includegraphics[width=0.49\linewidth]{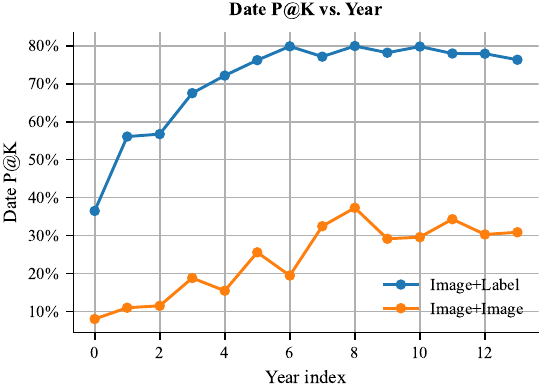}
    \includegraphics[width=0.49\linewidth]{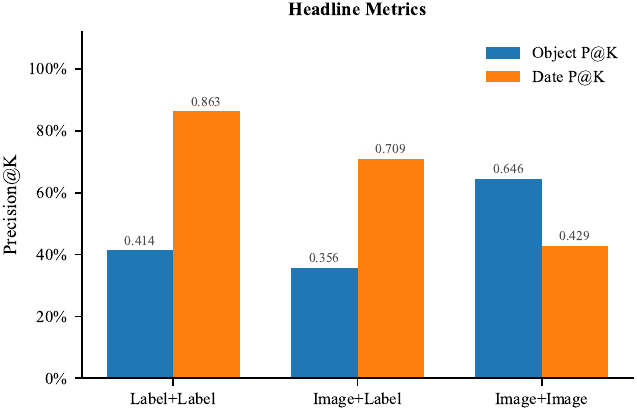}
    \caption{Left: date estimation precision per source year across inference modes. Both Image+Label and Image+Image exhibit degraded performance for earlier decades, with Image+Label remaining more stable across the temporal range. Right: scalar summary of object and date precision per inference mode. In Image+Image retrieval, the relationship between object and date precision inverts relative to label-based modes, indicating that temporal injection is a stricter operation than category preservation under the learned decomposition.}
    \label{fig:error_per_year}
\end{figure}

Figures~\ref{fig:dumbell_date} and~\ref{fig:dumbell_object} decompose the performance drop from Image+Label to Image+Image inference on a per-object basis, for date estimation and object retrieval respectively. Object retrieval degradation is more uniformly distributed across categories, suggesting that the category subspace remains largely stable under temporal transplantation for most classes, with isolated exceptions corresponding to categories with high inter-class visual overlap.

\begin{figure}[t]
    \centering
    \includegraphics[width=.879\linewidth]{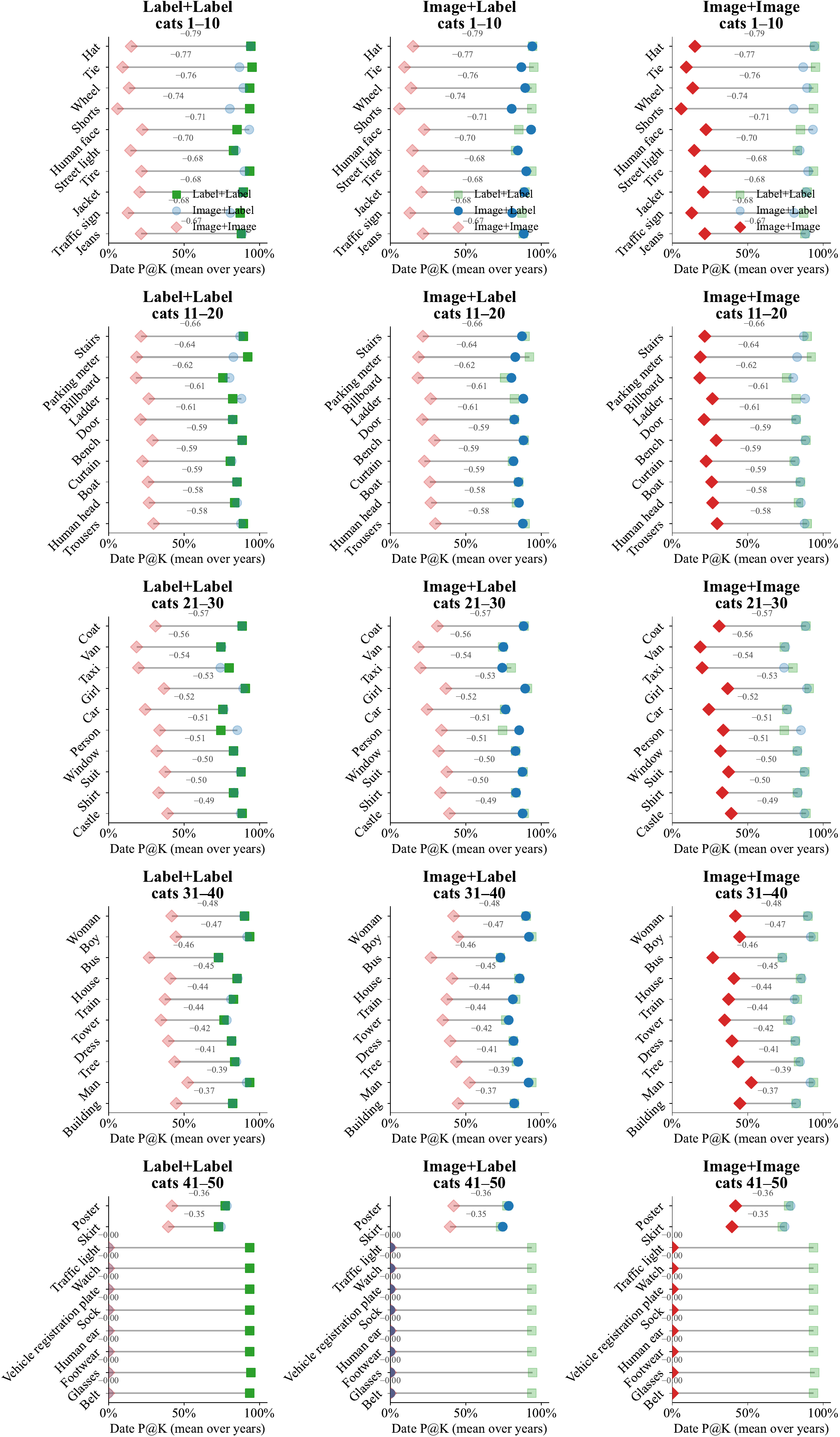}
    \caption{Per-object drop in date retrieval precision.}
    \label{fig:dumbell_date}
\end{figure}

\begin{figure}[t]
    \centering
    \includegraphics[width=.879\linewidth]{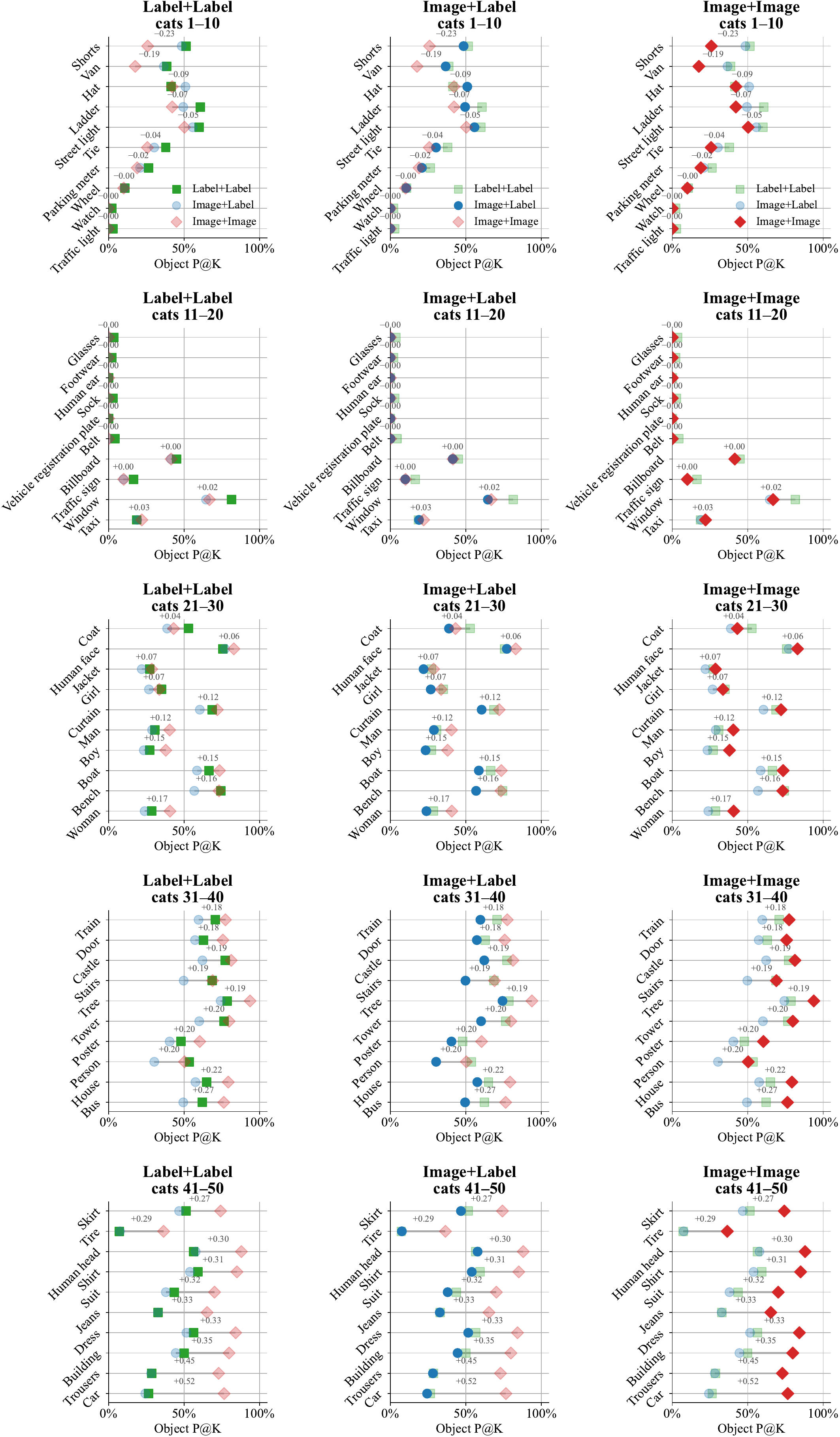}
    \caption{Per-object drop in object retrieval precision.}
    \label{fig:dumbell_object}
\end{figure}
\clearpage
\section{CLIP Baseline}
\label{sec:clip_baseline}
Both CLIP \cite{radford2021learning} baselines use OpenCLIP \cite{ilharco_gabriel_2021_5143773} ViT-B/32 \cite{dosovitskiy2020image} pre-trained on LAION-2B~\cite{Schuhmann2022LAION5BAO} as a frozen encoder, with no fine-tuning on DEW or on
any historical imagery. All image embeddings are $\ell_2$-normalised to unit norm prior to
any arithmetic operation or nearest-neighbour query, following the same convention as
TDIR. A single Annoy index~ \footnote{\url{https://github.com/spotify/annoy}} of $N_{\text{trees}}=50$ trees is built over the CLIP
vision embeddings of the full training split (identical images to those used for TDIR
training), so that retrieval pool and gallery are exactly matched across methods.
Nearest-neighbour search is performed in Euclidean space over the unit hypersphere,
which is equivalent to cosine similarity for $\ell_2$-normalised vectors. The top-$K=10$
neighbours are returned for every query, and precision is computed identically to
Section~\ref{sec:eval_protocol}.

\subsection{CLIP Arithmetic Baseline}
\label{sec:clip_baseline_b}

The arithmetic baseline replaces the learned proxy vectors $\mu^c$ and $\mu^y$ of TDIR
with CLIP text embeddings, exploiting the joint vision--language alignment acquired
during pre-training. For each object category $c$ we encode the prompt \emph{"A photo
of \{c\}"}, and for each target year $y$ (bucketed in 5-year lustrum bins from 1930 to
1999, yielding 14 bins) we encode \emph{"A photo in \{year\}"}. All text embeddings are
$\ell_2$-normalised after encoding, producing a category proxy matrix $T^c \in
\mathbb{R}^{C \times d}$ and a year proxy matrix $T^y \in \mathbb{R}^{Y \times d}$,
with $d=512$ for ViT-B/32. These matrices serve as drop-in substitutes for $P^c$ and
$P^y$ in all three inference modes of Section~\ref{sec:inference}, with no modification to the retrieval
procedure.
 
\paragraph{Label+Label.} The query vector is the un-normalised sum $q = T^c_c + T^y_y$,
directly mirroring the proxy-sum query $q = \mu^c + \mu^y$ of TDIR.
 
\paragraph{Image+Label.} The query is $q = V(x_c) + T^y_y$, where $V(x_c)$ is the
unit-norm CLIP vision embedding of the category image.
 
\paragraph{Image+Image (zero-label).} The reference image $x_y$ is sampled uniformly
from the training index. Its inferred category $\hat{c}$ is obtained by argmax cosine
similarity against all category text embeddings, $\hat{c} = \operatorname*{arg\,max}_{c'}
V(x_y)^\top T^c_{c'}$. The temporal residual is then extracted as $r_y = V(x_y) -
T^c_{\hat{c}}$, and the final query is $q = V(x_c) + r_y$.
 
The poor date precision reported in Table~\ref{tab:backbones} for this baseline is informative: despite
CLIP embeddings carrying rankable temporal information \cite{sonthalia2026rankability},
the arithmetic composition fails to disentangle temporal and categorical signals because
CLIP was never trained to produce orthogonal category and year subspaces. The residual
$r_y$ retains substantial categorical content, and the year text prompts encode only a
weak and diffuse notion of historical period that does not translate into a geometrically
consistent displacement on the visual embedding manifold. This confirms that temporal
composability is not an emergent property of vision--language pre-training, but requires
the explicit geometric structure imposed by TDIR.

\subsection{CLIP Prompting Baseline}
\label{sec:clip_baseline_a}

The prompting baseline replaces vector arithmetic with a two-stage retrieve-and-rerank
strategy. Rather than composing both signals into a single query vector, retrieval is
first performed using the primary modality alone, and the resulting candidate pool is
then re-ranked by cosine similarity to the secondary modality signal. This avoids
relying on cross-modal vector addition and instead lets each signal operate
independently on the unit hypersphere.
 
Concretely, a candidate pool of size $P=100$ is retrieved by Annoy nearest-neighbour
search using the primary modality embedding. The stored embeddings of all pool
candidates are then re-scored by their cosine similarity to the secondary modality
embedding, and the top-$K=10$ candidates after re-scoring are returned.
 
\paragraph{Label+Label.} The candidate pool is retrieved by $T^c_c$ alone; pool
candidates are then re-ranked by cosine similarity to $T^y_y$.
 
\paragraph{Image+Label.} The candidate pool is retrieved by $V(x_c)$; re-ranking uses
$T^y_y$. The pool is computed once per test image and reused for all 14 year bins,
amortising the retrieval cost across the full label set.
 
\paragraph{Image+Image.} The candidate pool is retrieved by $V(x_c)$; re-ranking uses
$V(x_y)$ directly, replacing the residual $r_y$ with the raw reference image embedding.
This sidesteps the category-inference step and provides a stronger re-ranking signal
when category and date co-vary in the reference image.
 
As shown in Table~\ref{tab:backbones}, the prompting baseline recovers competitive object precision
across all three modes, confirming that CLIP vision embeddings carry strong categorical
signal. Date precision, however, remains substantially below TDIR for all inference
modes even under reranking. This gap persists because reranking by $T^y_y$ or $V(x_y)$
can only promote candidates that already appear in the pool retrieved by the category
signal, and nothing in the CLIP representation space guarantees that temporally
relevant images rank highly under a category-only query. The contrast between the
two baselines and TDIR thus isolates the contribution of the learned orthogonal
decomposition: without it, object and temporal retrieval remain fundamentally
entangled regardless of the composition strategy employed.

\section{Frequently Asked Questions (FAQs)}
In this section, readers will find a series of answers to questions raised during peer review, discussions, and others. We hope they cana be useful in interpreting the data and results presented in the paper, as well as the method itself.
\subsection{Joint metric: Def.~1's independence assumption holds empirically.}
We added Precision@K requiring \emph{both} category and date correct simultaneously to all three retrieval modes, letting us test directly (rather than assume) whether treating category and date as independent (Def.~1) actually holds at inference. On Label+Label (K=10, a stratified 7\,407-crop pool):
\begin{center}
\textbf{Obj P@10 = 0.488 \quad Date P@10 = 0.718 \quad Joint P@10 = 0.349}
\end{center}
\noindent against $0.488\times0.718=0.350$ predicted under independence, a $0.4\%$ relative gap well within noise. Object- and date-correctness are therefore empirically independent: neither axis is traded off against, nor rides for free on, the other, exactly the behaviour Def.~1 and Thm.~1's joint optimisation are designed to produce.

\subsection{Proxy geometry corroborates this at the representation level.}
We further computed the full cosine matrix between the 50 category and 14 year proxies of the reported ConvNeXt+TDIR checkpoint (its cached index reproduces Table~2/3 exactly). Cross-block mean $|\cos|=0.204$, close to the $0.025$ floor for random unit vectors in $\mathbb{R}^{1024}$ at this scale. Fig.~1 shows both regimes side by side across \emph{all 14} year proxies: 27/50 categories (\eg\ \textit{Window}, \textit{Bus}) sit at that floor throughout, genuinely orthogonal as Thm.~1 predicts in the exact case. A smaller set of $\sim$15 categories with a plausible intrinsic appearance-era link (\textit{tie, belt, watch, sock, footwear,\ldots}) shows the bounded, data-dependent residual $\delta$ that Prop.~2 anticipates for the realistic case: \eg\ \textit{Footwear} peaks at $+0.87$ in 1930 and is 3--6$\times$ weaker by 1965 on. This residual tracks training-set size almost exactly (per-year mean $|\cos|$ vs.\ log sample count: $r=-0.90$, $p{<}10^{-4}$; 1930: 14k crops $\rightarrow 0.44$; 1985: 108k crops $\rightarrow 0.11$), giving $\delta$ a precise, measured, localized explanation rather than an abstract constant.

\begin{figure}[t]
  \centering
  \includegraphics[width=\linewidth]{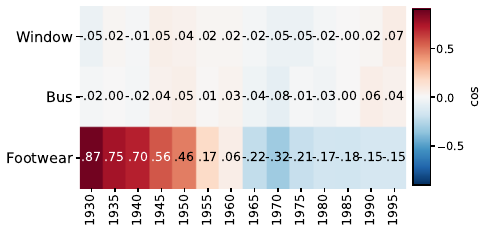}
  \vspace{-6pt}
  \caption{Category-year cosine similarity across \emph{all 14} year proxies for the two most orthogonal (\textit{Window}, \textit{Bus}) and the most entangled (\textit{Footwear}) of the 50 categories.}
  \label{fig:matrix}
  \vspace{-4pt}
\end{figure}

\subsection{Baselines and Previous Work}
\paragraph{Modern SSL pretraining (DINOv3).}
Because DINOv3 is not a specific architecture but a pretraining regime, we have loaded DINOv3 weights into the same ConvNeXt-Base backbone used throughout (Table~2) and train under the identical 50-category protocol. At epoch 1 already, cross-block mean $|\cos|=0.173$, already at or below the fully-converged ImageNet-supervised value ($0.204$), with the same ordinal decay replicating: the decomposition is therefore not an artifact of ImageNet-supervised pretraining.

\paragraph{On the suggested \cite{molina2022generic,net2024transformer} baselines.}
We are well aware of both \cite{molina2022generic}'s smooth-nDCG loss needs a relevance function derived from a numerical variable: it can replace our proxy loss on the year component, but never on the category, so it does not serve as a composed-retrieval loss. However, further results may incorporate such loss as date proxy loss, but the method itself is not capable of composing concepts. And, on the other hand, \cite{net2024transformer} is incompatible with our setup by any means: it runs object detection over multi-object images and requires a dedicated encoder per object label to predict a date (in \cite{net2024transformer} the category is not learned, it is passed as input).

\subsection{Ordinal structure \& unseen years.}
Cosine similarity between two year proxies decays smoothly with their temporal gap, although only pairwise transitivity (Def.~2) was ever enforced at training time:
\begin{center}
\vspace{-2pt}
\includegraphics[width=\linewidth]{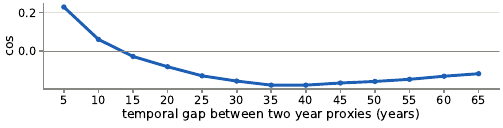}
\vspace{-8pt}
\end{center}
Leave-one-out check: a year's true proxy is consistently closer to the midpoint of its two neighbours than to either neighbour alone (\eg\ $\cos=0.66$ vs.\ $0.57$ at 1935; all 12 interior years), so the subspace already extrapolates locally.

\paragraph{Objects with weak temporal signature.}
The categories with the lowest year-entanglement above (\textit{window, bus, door, boat, tower}) are the slow-changing, architecture-type categories that one may flag as barely changing across decades. This is actually consistent, not contradictory, since their proxies correctly encode little date information because their appearance carries little.

\end{document}